\documentclass[journal]{IEEEtran}
\usepackage{amssymb}
\usepackage{amsmath}
\usepackage{color}
\usepackage[colorinlistoftodos,prependcaption,textsize=tiny]{todonotes}
\usepackage{subfig}
\usepackage{graphicx}
\usepackage{svg}
\usepackage{hyperref}
\usepackage{array}
\usepackage{makecell}
\usepackage{pifont}
\usepackage{amsthm}
\usepackage{enumitem}
\usepackage[export]{adjustbox}
\usepackage[flushleft]{threeparttable}
\usepackage{booktabs}
\usepackage{mathtools}
\usepackage{float}
\PassOptionsToPackage{hyphens}{url}\usepackage{hyperref}
\hypersetup{
    colorlinks=true,
    linkcolor=blue,
    filecolor=magenta,      
    urlcolor=cyan,
    }

\usepackage{mathtools}
\mathtoolsset{showonlyrefs} 
\newcommand{\cmark}{\ding{51}}%
\newcommand{\xmark}{\ding{55}}%

\definecolor{mygray}{RGB}{200, 200, 200}
\definecolor{mygreen}{RGB}{1, 200, 1}
\definecolor{myblue}{RGB}{1, 1, 255}
\definecolor{mypurple}{RGB}{200, 1, 200}
\definecolor{lightgray}{gray}{0.8}
\definecolor{lightlightgray}{gray}{0.9}
\DeclareMathOperator*{\argmax}{arg\,max}
\DeclareMathOperator*{\argmin}{arg\,min}

\newtheorem{theorem}{Theorem}
\newtheorem{proposition}{Proposition}
\theoremstyle{definition}
\newtheorem{definition}{Definition}
\hypersetup{urlcolor=blue}
\makeatletter
\DeclareUrlCommand\ULurl@@{%
  \def\UrlLeft{\bgroup}%
  \def\UrlRight{\egroup}}
\def\ULurl@#1{\hyper@linkurl{\ULurl@@{#1}}{#1}}
\DeclareRobustCommand*\ULurl{\hyper@normalise\ULurl@}
\makeatother

\begin{document}
\title{Stable Motion Primitives via Imitation and Contrastive Learning}

\author{Rodrigo~P\'erez-Dattari,~\IEEEmembership{Member,~IEEE,}
        Jens~Kober,~\IEEEmembership{Senior~Member,~IEEE}
}


\maketitle

\begin{abstract}
Learning from humans allows non-experts to program robots with ease, lowering the resources required to build complex robotic solutions. Nevertheless, such data-driven approaches often lack the ability to provide guarantees regarding their learned behaviors, which is critical for avoiding failures and/or accidents.
In this work, we focus on reaching/point-to-point motions, where robots must always reach their goal, independently of their initial state. This can be achieved by modeling motions as dynamical systems and ensuring that they are globally asymptotically stable. 
Hence, we introduce a novel Contrastive Learning loss for training Deep Neural Networks (DNN) that, when used together with an Imitation Learning loss, enforces the aforementioned stability in the learned motions. Differently from previous work, our method does not restrict the structure of its function approximator, enabling its use with arbitrary DNNs and allowing it to learn complex motions with high accuracy. 
We validate it using datasets and a real robot. In the former case, motions are 2 and 4 dimensional, modeled as first- and second-order dynamical systems. In the latter, motions are 3, 4, and 6 dimensional, of first and second order, and are used to control a 7DoF robot manipulator in its end effector space and joint space. More details regarding the real-world experiments are presented in: \url{https://youtu.be/OM-2edHBRfc}.
\end{abstract}

\begin{IEEEkeywords}
Imitation Learning, Contrastive Learning, Dynamical Systems, Motion Primitives, Deep Neural Networks 
\end{IEEEkeywords}

\IEEEpeerreviewmaketitle

\section{Introduction}
Imitation Learning (IL) provides a framework that is intuitive for humans to use, without requiring them to be robotics experts. It allows robots to be programmed by employing methods similar to the ones humans use to learn from each other, such as demonstrations, corrections, and evaluations. This significantly reduces the resources needed for building robotic systems, making it particularly appealing for real-world applications (e.g., Fig. \ref{fig:cover}).

Nevertheless, due to their data-driven nature, IL methods often lack guarantees, such as ensuring that a robot's motion always reaches its target, independently of its initial state (e.g., Fig. \ref{fig:bc_example_DS}). This can be a major limitation for implementing methods in the real world since it can lead to failures and/or accidents.

To tackle these challenges, we can model motions as dynamical systems whose evolution describes a set of human demonstrations \cite{ijspeert2013dynamical, khansari2011learning, rana2020euclideanizing}. This is advantageous because 1) the model depends on the robot’s state and it is learned offline, enabling the robot to adapt to changes in the environment during task execution, and 2) dynamical systems theory can be employed to analyze the behavior of the motion and provide guarantees.

\begin{figure}[t]
    \centering
    \includegraphics[width=0.8\columnwidth]{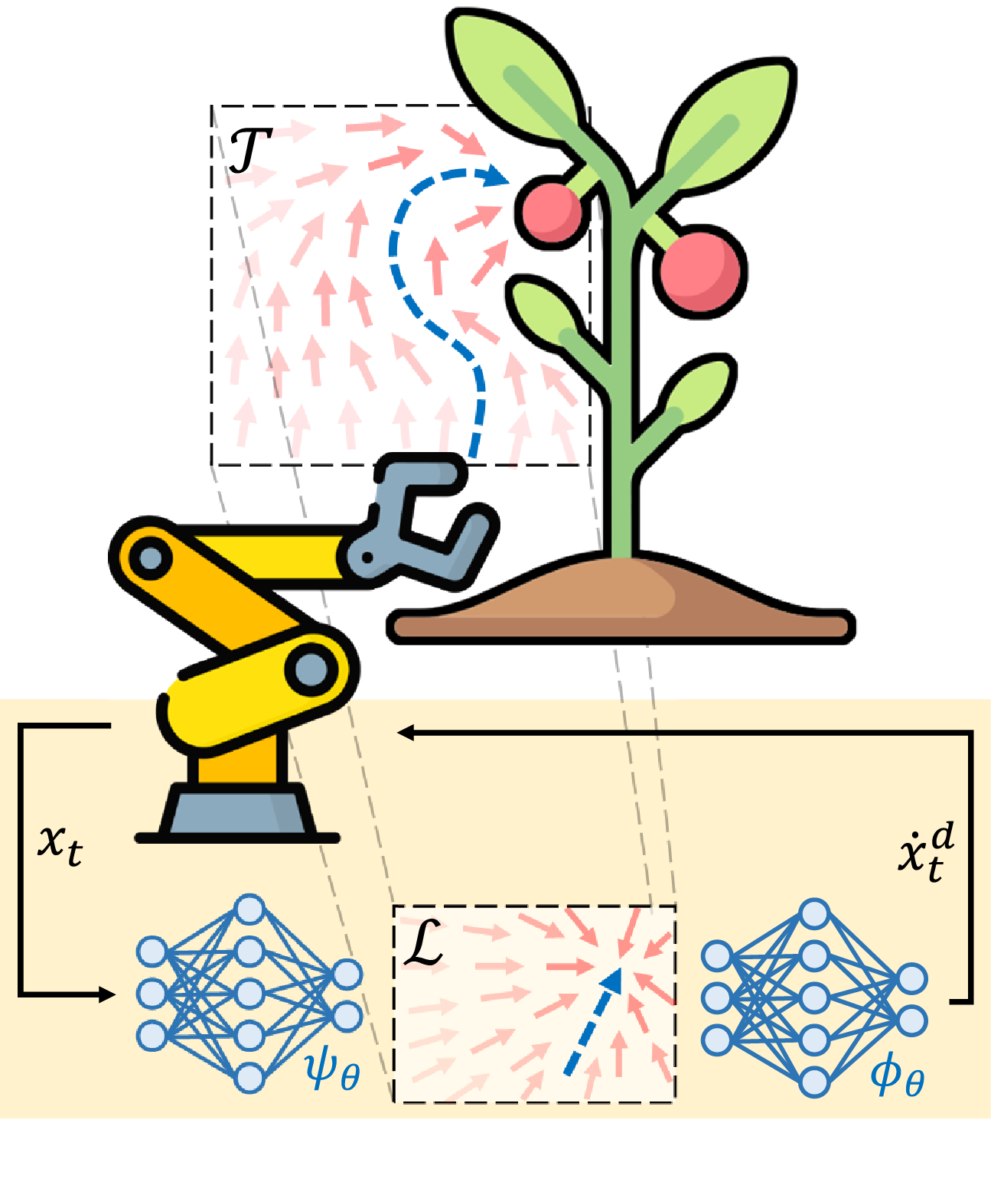}
    \caption{Overview of a motion model learned using the proposed framework. The blue trajectory in the task space $\mathcal{T}$, shows the movement of the robot's end effector when starting from its current state $x_{t}$. The evolution of this trajectory is defined by the dynamical system $\dot{x}^{d}_{t}=\phi_{\theta}(\psi_{\theta}(x_{t}))$, which is represented with a vector field of red arrows in the rest of the space. Using Contrastive Learning, this system is coupled with a well-understood and stable dynamical system in the latent space $\mathcal{L}$ that ensures its stability.  }
    \label{fig:cover}
\end{figure}

In this work, we focus on learning dynamical systems from demonstrations to model \emph{reaching motions}, as a wide range of tasks requires robots to reach goals, e.g., hanging objects, pick-and-place of products, crop harvesting, and button pressing. Furthermore, these motions can be sequenced to model cyclical behaviors, extending their use for such problems as well \cite{kober2012playing}.

A reaching motion modeled as a dynamical system is considered to be globally asymptotically stable if the robot always reaches its goal, independently of its initial state. In this work, we will refer to such systems as \emph{stable} for short. Notably, by employing dynamical systems theory, stability in reaching motions can be enforced when learning from demonstrations, providing guarantees to these learning frameworks.

In the literature, there is a family of works that use this approach to learn stable motions from demonstrations \cite{khansari2011learning, rana2020euclideanizing, sindhwani2018learning}. However, these often constrain the structure of their learning models to meet certain conditions needed to guarantee the stability of their motions, e.g., by enforcing the learning functions to be invertible \cite{perrin2016fast,rana2020euclideanizing,urain2020imitationflow} or positive/negative definite \cite{khansari2011learning,khansari2014learning,lemme2014neural}. Although these constraints ensure stability, they limit the applicability of the methods to a narrow range of models. For example, if a novel promising Deep Neural Network (DNN) architecture is introduced in the literature, it would not be straightforward/possible to use it in such frameworks, since, commonly, DNN architectures do not have this type of constraints. Furthermore, the learning flexibility of a function approximator is limited if its structure is restricted, which can hurt its accuracy performance when learning motions.

\begin{figure}[t]
    \centering
    \includegraphics[width=0.8\columnwidth]{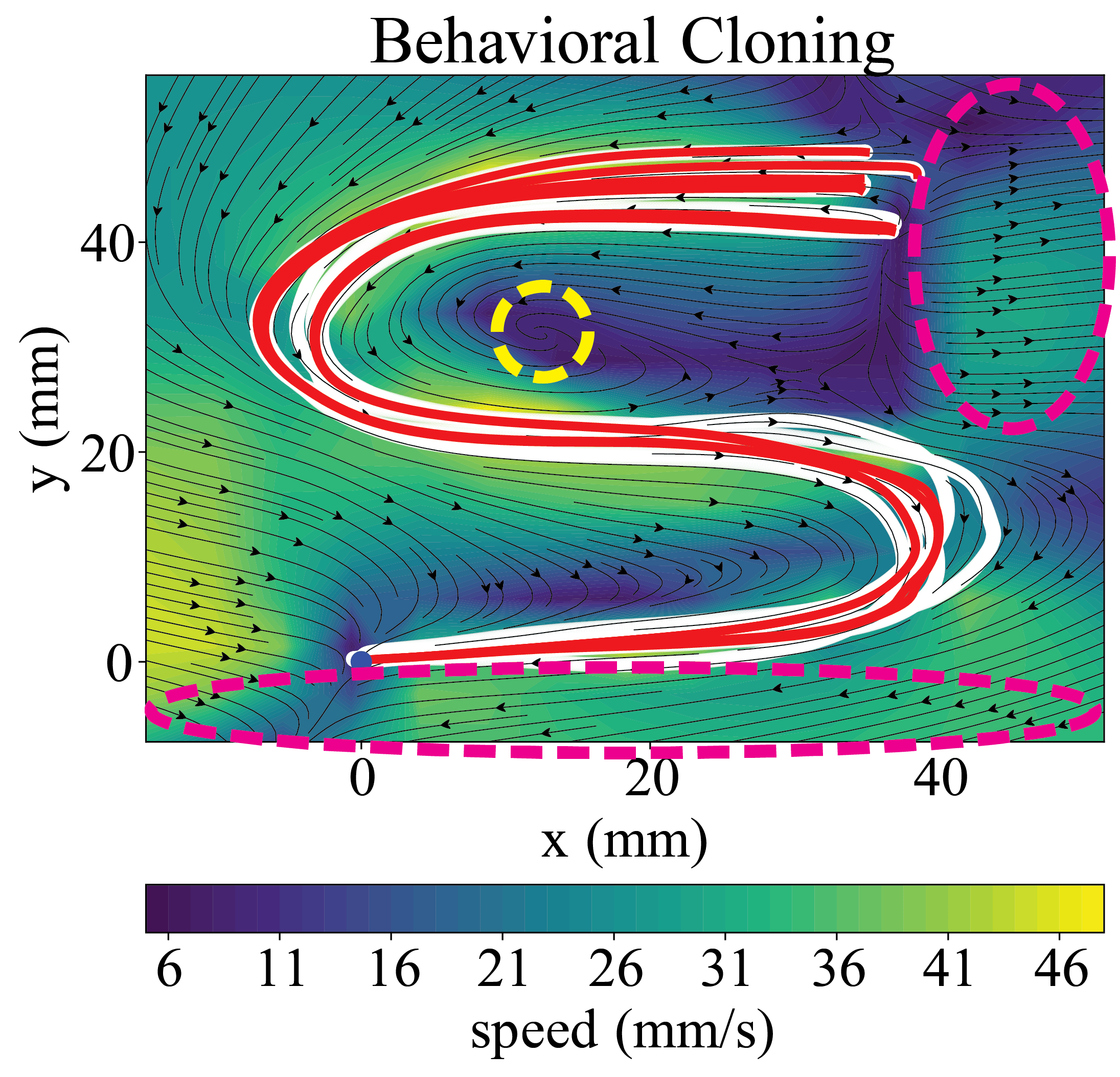}
    \caption{Example of a motion learned using Behavioral Cloning. White curves represent demonstrations. Red curves represent learned motions when starting from the same initial positions as the demonstrations. The arrows indicate the vector field of the learned dynamical system. The yellow dotted line shows a region with a spurious attractor. The magenta lines show regions where the trajectories diverge away from the goal.}
    \label{fig:bc_example_DS}
\end{figure}

Hence, in the context of DNNs\footnote{We understand DNNs as a collection of machine learning algorithms that learn in a hierarchical manner, i.e., the function approximator consists of a composition of multiple functions, and are optimized by means of backpropagation \cite{awad2015efficient}.}, we propose a novel method for learning stable motions without constraining the structure of the function approximator. To achieve this, we introduce a contrastive loss \cite{chopra2005learning,kaya2019deep} to enforce stability in dynamical systems modeled with arbitrary DNN architectures. This is achieved by transferring the stability properties of a simple, stable dynamical system to the more complex system that models the demonstrations (see Figure \ref{fig:cover}). To the best of our knowledge, this is the first approach that learns to generate stable motions with DNNs without relying on a specific architecture type.

We validate our method using both simulated and real-world experiments, demonstrating its ability to successfully scale in terms of the order and dimensionality of the dynamical system. Furthermore, we showcase its capabilities for controlling a 7DoF robotic manipulator in both joint and end-effector space. Lastly, we explore potential extensions for the method, such as combining motions by learning multiple systems within a single DNN architecture.

The paper is organized as follows: related works are presented in Section \ref{sec:related_works}. Section \ref{sec:preliminaries} describes the background and problem formulation of our method. Section \ref{sec:method} develops the theory required to introduce the contrastive loss, introduces it, and explains how we employ it in the context of Imitation Learning. Experiments and results are divided into Sections \ref{sec:simulated_experiments}, \ref{sec:real_world}, and \ref{sec:extensions}. Section \ref{sec:simulated_experiments} validates our method using datasets of motions modeled as first-order and second-order dynamical systems. These motions are learned from real data, but they are evaluated without employing a real system. Section \ref{sec:real_world} validates the method in a real robot, and Section \ref{sec:extensions} studies possible ways of extending it. Finally, the conclusions are drawn in Section \ref{sec:conclusions}.
\section{Related Works}\label{sec:related_works}
Several works have approached the problem of learning motions modeled as dynamical systems from demonstrations while ensuring their stability. By observing if these works employ either \emph{time-varying} or \emph{time-invariant} dynamical systems, we can divide them into two groups. In time-varying dynamical systems, the evolution of the system explicitly depends on time (or a phase). In contrast, time-invariant dynamical systems do not depend on time \emph{directly}, but only through its time-varying input (i.e., the state of the system). The property of a system being either time-invariant or not, conditions the type of strategies that can be employed to enforce its stability. Hence, for this work, it makes sense to make a distinction between these systems. 

One seminal work of IL that addresses stability for time-varying dynamical systems introduces Dynamical Movement Primitives (DMPs) \cite{ijspeert2013dynamical}. This method takes advantage of the time-dependency (via the phase of the \emph{canonical system}) of the dynamical system to enforce its nonlinear part, which captures the behavior of the demonstration, to vanish as time goes to infinity. Then, they build the remainder of the system to be a function that is well-understood and stable by construction. Hence, since the nonlinear part of the motion will eventually vanish, its stability can be guaranteed. In the literature, some works extend this idea with probabilistic formulations \cite{li2022prodmps,amor2014interaction}, and others have extended its use to the context of DNNs \cite{pervez2017learning,ridge2020training,bahl2020neural}. 

These time-varying dynamical system approaches are well-suited for when the target motions have clear temporal dependencies. However, they can generate undesired behaviors when encountering perturbations (assuming the time/phase is not explicitly modulated), and they lack the ability to model different behaviors for different regions of the robot's state space. In contrast, time-invariant dynamical systems can easily address these shortcomings, but they can be more challenging to employ when motions contain strong temporal dependencies. Therefore, IL formulations with such systems are considered to be complementary to the ones that employ time-varying systems \cite{calinon2011encoding,ravichandar2020recent}. In this work, we focus on time-invariant dynamical systems.

An important family of works has addressed the problem of modeling stable motions as time-invariant dynamical systems. These approaches often constrain the structure of the dynamical systems to ensure Lyapunov stability by design. In this context, one seminal work introduces the Stable Estimator of Dynamical Systems (SEDS) \cite{khansari2011learning}. This approach imposes constraints on the structure of Gaussian Mixture Regressions (GMR), ensuring stability in the generated motions.

Later, this idea inspired other works to explicitly learn Lyapunov functions that are consistent with the demonstrations and \emph{correct} the transitions of the learned dynamical system such that they are stable according to the learned Lyapunov function \cite{khansari2014learning,lemme2014neural,duan2017fast}. Furthermore, several extensions of SEDS have been proposed, for instance by using physically-consistent priors \cite{figueroa2018physically}, contraction theory \cite{ravichandar2017learning} or diffeomorphisms \cite{neumann2015learning}. 

Moreover, some of these ideas, such as the use of contraction theory or diffeomorphisms have also been used outside the scope of SEDS. Contraction theory ensures stability by enforcing the distance between the trajectories of a system to reduce, according to a given metric, as the system evolves. Hence, it has been employed to learn stable motions from demonstrations \cite{ravichandar2015learning,blocher2017learning}. In contrast, diffeomorphisms can be employed to transfer the stability properties of a stable and well-understood system, to a complex nonlinear system that models the behavior of the demonstrations. Hence, this strategy has also been employed to learn stable motions from demonstrations \cite{perrin2016fast,rana2020euclideanizing,urain2020imitationflow}. As we explain in Section \ref{sec:connection_diffeo}, our method is closely related to these approaches. It is worth noting that of the mentioned strategies, only \cite{urain2020imitationflow} models stable stochastic dynamics. However, this concept could also be explored with other encoder-decoder stochastic models, e.g., \cite{li2022prodmps}.

Understandably, all of these methods constrain some part of their learning framework to ensure stability. From one point of view, this is advantageous, since they can guarantee stability. However, in many cases, this comes with the cost of reducing the flexibility of the learned motions (i.e., loss in accuracy). Notably, some recent methods have managed to reduce this loss in accuracy \cite{rana2020euclideanizing,urain2020imitationflow}; however, they are still limited in terms of the family of models that can be used with these frameworks, which harms their scalability. Consequently, in this work, we address these limitations by enforcing the stability of the learned motions as a soft constraint and showing its effectiveness in obtaining stable, accurate, and scalable motions.
\section{Preliminaries}
\label{sec:preliminaries}
\subsection{Dynamical Systems as Movement Primitives}
In this work, we model motions as nonlinear time-invariant dynamical systems defined by the equation 
\begin{equation}
    \dot{x} = f(x),
\label{DS1}
\end{equation}
where $x \in \mathbb{R}^{n}$ is the system's state and $f: \mathbb{R}^{n} \to \mathbb{R}^{n}$ is a nonlinear continuous and continuously differentiable function. The evolution defined by this dynamical system is transferred to the robot's state by tracking it with a lower-level controller.

\subsection{Global Asymptotic Stability}
We are interested in solving reaching tasks. From a dynamical system perspective, this means that we want to construct a system where the goal state $x_{g} \in \mathbb{R}^{n}$ is a \emph{globally asymptotically stable} equilibrium. An equilibrium $x_{g}$ is globally asymptotically stable if $ \forall x \in \mathbb{R}^{n}$, 
\begin{equation}\label{eq:stability}
    \lim_{t \rightarrow \infty} || x - x_{g}||=0.
\end{equation}
Note that for this condition to be true, the time derivative of the dynamical system at the attractor must be zero, i.e., $\dot{x}=f(x_{g})=0$. 

For simplicity, we use the word \emph{stable} to refer to these systems.

\subsection{Problem Formulation}
Consider the scenario where a robot aims to learn a reaching motion, in a given space $\mathcal{T} \subset \mathbb{R}^{n}$ and with respect to a given goal $x_{g} \in \mathcal{T}$, based on a set of demonstrations $\mathcal{D}$. The robot is expected to imitate the behavior shown in the demonstrations while always reaching $x_{g}$, regardless of its initial state.

The dataset $\mathcal{D}$ contains $N$ demonstrations in the form of trajectories $\tau_{i}$, such that $\mathcal{D} = (\tau_{0}, \tau_{1}, ..., \tau_{N-1})$. Each one of these trajectories contains the evolution of a dynamical system with discrete-time states $x_{t} \in \mathcal{T}$ when starting from an initial state $x_{0}$ and it transitions for $T$ time steps $t$ of size $\Delta t$. Hence, $\tau_{i} = (x_{0}^{i}, x_{1}^{i}, ..., x_{T-1}^{i})$, where $T$ does not have to be the same for every demonstration, and here we added the superscript $i$ to the states to explicitly indicate that they belong to the trajectory $\tau_{i}$. Note, however, that the state superscript will not be used for the remainder of the paper.

We assume that these trajectories are drawn from the distribution $p^{*}(\tau)$, where every transition belonging to a trajectory sampled from this distribution follows the \emph{optimal} (according to the demonstrator's judgment) dynamical system $f^{*}$. On the other hand, the robot's motion is modeled as the parametrized dynamical system $f^{\mathcal{T}}_{\theta}$, which induces the trajectory distribution $p_{\theta}(\tau)$, where $\theta$ is the parameter vector.

Then, the objective is to find $\theta^{*}$ such that the distance between the trajectory distributions induced by the human and learned dynamical system is minimized while ensuring the stability of the motions generated with $f^{\mathcal{T}}_{\theta}$ towards $x_{g}$.
This can be formulated as the minimization of the (forward) Kullback-Leibler divergence between these distributions \cite{bishop2006pattern}, subject to a stability constraint of the learned system:
\begin{subequations}\label{eqn:problem_formulation}
\begin{align}
\theta^* = \argmin_{\theta} \quad & D_{\text{KL}}\left(p^{*}(\tau) ||  p_{\theta}(\tau)\right)\\ 
\text { s.t. } &\lim_{t \rightarrow \infty} || x_{t} - x_{g}||=0,\\
 &\forall x_{t} \in \mathcal{T} \text{ evolving with } f^{\mathcal{T}}_{\theta}. \nonumber
\label{eq:dynamic_constraints}
\end{align}
\end{subequations}
\section{Methodology}
\label{sec:method}
We aim to learn motions from demonstrations modeled as nonlinear time-invariant dynamical systems. 
In this context, we present the \emph{CONvergent Dynamics from demOnstRations} (CONDOR) framework. This framework learns the parametrized function $f^{\mathcal{T}}_{\theta}$ using human demonstrations and ensures that this dynamical system has a globally asymptotically stable equilibrium at $x_{g}$ while being accurate w.r.t. the demonstrations. 

To achieve this, we extend the Imitation Learning (IL) problem with a novel loss $\ell_{\text{stable}}$ based on Contrastive Learning (CL) \cite{chopra2005learning} that aims to ensure the stability of the learned system. Hence, if the IL problem minimizes the loss $\ell_{\text{IL}}$, our framework minimizes 
\begin{equation}
    \ell_{\text{CIL}} = \ell_{\text{IL}} + \lambda \ell_{\text{stable}},
\end{equation}
where $\lambda \in \mathbb{R}$ is a weight. We refer to $\ell_{\text{CIL}}$ as the \emph{Contrastive Imitation Learning} (CIL) loss.

\subsection{Structure of CONDOR}\label{sec:condor_structure}
\begin{figure}[t]
    \centering
    \includegraphics[width=\columnwidth]{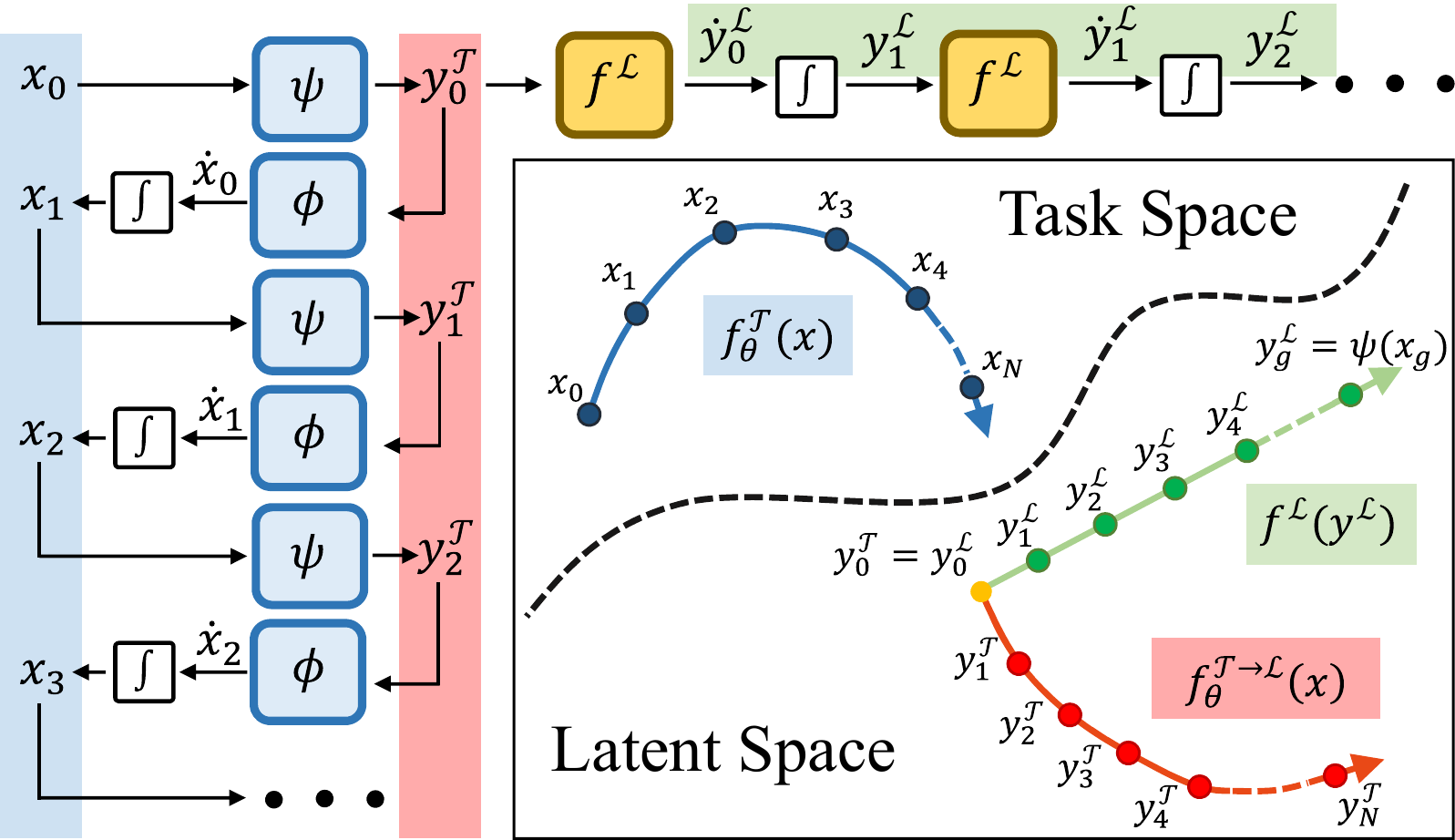}
    \caption{Structure of CONDOR. We show an example of discrete-time trajectories generated with $f_{\theta}^{\mathcal{T}}$, $f^{\mathcal{L}}$ and $f_{\theta}^{\mathcal{T}\to\mathcal{L}}$ \textbf{before training} the DNN. Starting from an initial point $x_{0}$, a trajectory is generated in $\mathcal{T}$ using $f_{\theta}^{\mathcal{T}}$ (blue) and two trajectories are generated in $\mathcal{L}$. One of them follows $f_{\theta}^{\mathcal{T}\to\mathcal{L}}$ (red), and the other follows $f^{\mathcal{L}}$ (green).}
    \label{fig:tasklatent}
\end{figure}
The objective of $\ell_{\text{stable}}$ is to ensure that $f^{\mathcal{T}}_{\theta}$ shares stability properties with a simple and well-understood system. We will refer to this system as $f^{\mathcal{L}}$, which is designed to be stable by construction. Consequently, if $f^{\mathcal{L}}$ is stable, then $f^{\mathcal{T}}_{\theta}$ will also be stable. 

Since $f^{\mathcal{T}}_{\theta}$ is parametrized by a DNN, we can define $f^{\mathcal{L}}$ to reside within the output of one of the hidden layers of $f^{\mathcal{T}}_{\theta}$. This formulation might seem arbitrary; however, it will be shown later that it enables us to introduce the \emph{stability conditions}, which serve as the foundation for designing $\ell_{\text{stable}}$. Therefore, we define the dynamical system $f^{\mathcal{T}}_{\theta}$ as a composition of two functions, $\psi_{\theta}$ and $\phi_{\theta}$,
\begin{equation}
    \dot{x}_{t}=f^{\mathcal{T}}_{\theta}(x_{t}) = \phi_{\theta}(\psi_{\theta}(x_{t})),
\end{equation}
$\forall x_{t} \in \mathcal{T}$. Note that $f^{\mathcal{T}}_{\theta}$ is a standard DNN with $L$ layers. $\psi_{\theta}$ denotes layers $1...l$, and $\phi_{\theta}$ layers $l+1...L$. We define the output of layer $l$ as the latent space $\mathcal{L} \subset \mathbb{R}^{n}$. Moreover, for simplicity, although we use the same $\theta$ notation for both $\psi_{\theta}$ and $\phi_{\theta}$, each symbol actually refers to a different subset of parameters within $\theta$. These subsets together form the full parameter set in $f^{\mathcal{T}}_{\theta}$.

Then, the dynamical system $f^{\mathcal{L}}$ is defined to evolve within $\mathcal{L}$\footnote{Note that before training, this system will not completely reside in $\mathcal{L}$, since it is allowed to evolve outside the image of $\psi_{\theta}$.}. This system is constructed to be stable at the equilibrium $y_{g}=\psi(x_{g})$, and can be described by
\begin{equation}
    \dot{y}_{t}^{\mathcal{L}} = f^{\mathcal{L}}(y_{t}^{\mathcal{L}}),
\end{equation}
$\forall y^{\mathcal{L}}_{t} \in \mathcal{L}$. Here, $y_{t}^{\mathcal{L}}$ corresponds to the \emph{latent state variables} that evolve according to $f^{\mathcal{L}}$.

Lastly, it is necessary to introduce a third dynamical system. This system represents the evolution in $\mathcal{L}$ of the states visited by 
$f^{\mathcal{T}}_{\theta}$ when mapped using $\psi_{\theta}$, which yields the relationship
\begin{equation}
    \dot{y}_{t}^{\mathcal{T}} = f_{\theta}^{\mathcal{T}\to\mathcal{L}}(x_{t}) = \frac{\partial \psi_{\theta}(x_{t})}{\partial t},
\end{equation}
$\forall y^{\mathcal{T}}_{t} \in \mathcal{L}$, where $y_{t}^{\mathcal{T}}$ corresponds to the latent variables that evolve according to $f_{\theta}^{\mathcal{T}\to\mathcal{L}}$.

Figure \ref{fig:tasklatent} summarizes the introduced dynamical systems.

\subsection{Stability Conditions}
The above-presented dynamical systems allow us to introduce the stability conditions. These conditions state that if $f_{\theta}^{\mathcal{T}\to\mathcal{L}}$ exhibits the same behavior as $f^{\mathcal{L}}$, and only $x_{g}$ maps to $\psi_{\theta}(x_{g})$, then $f_{\theta}^{\mathcal{T}}$ is stable. We formally introduce them as follows:
\begin{theorem}[Stability conditions]
\label{theo:stable}
Let $f_{\theta}^{\mathcal{T}}$, $f_{\theta}^{\mathcal{T} \to \mathcal{L}}$ and $f^{\mathcal{L}}$ be the dynamical systems introduced in Section \ref{sec:condor_structure}. Then, $x_{g}$ is a globally asymptotically stable equilibrium of $f_{\theta}^{\mathcal{T}}$ if, $\forall x_{t} \in \mathcal{T}$,:
\begin{enumerate}[font=\itshape]
    \item $f_{\theta}^{\mathcal{T}\to\mathcal{L}}(x_{t})=f^{\mathcal{L}}(y_{t}^{\mathcal{T}})$,
    \item $\psi_{\theta}(x_{t})=y_{g} \Rightarrow x_{t}=x_{g}$.
\end{enumerate}
\end{theorem}
\begin{proof}
Since $f^{\mathcal{L}}$ is globally asymptotically stable at $y_{g}$, condition 1) indicates that as $t \to \infty$, $y_{t}^{\mathcal{T}}=y_{t}^{\mathcal{L}} \to y_{g}$. However, from condition 2) we know that $y^{\mathcal{T}}_{t}=\psi_{\theta}(x_{t})=y_{g}$ is only possible if $x_{t}=x_{g}$. Hence, as $t \to \infty$, $x_{t} \to x_{g}$. Then, $x_{g}$ is globally asymptotically stable in $f_{\theta}^{\mathcal{T}}(x_{t})$.
\end{proof}

Consequently, we aim to design $\ell_{\text{stable}}$ such that it enforces the stability conditions in the presented dynamical systems by optimizing $\psi_{\theta}$ and $\phi_{\theta}$.

\subsubsection{Connection with Diffeomorphism-based methods}\label{sec:connection_diffeo}
It is interesting to note that the stability conditions make $\psi_{\theta}$ converge to a \emph{diffeomorphism} between $\mathcal{T}$ and $\mathcal{L}$ (proof in Appendix \ref{appendix:diffeo_proof}). Consequently, our approach becomes tightly connected to methods that ensure stability using diffeomorphic function approximators \cite{perrin2016fast,rana2020euclideanizing,urain2020imitationflow}. However, differently from these methods, we do not require to take into account the structure of the function approximator and explicit relationships between $f_{\theta}^{\mathcal{T}}$ and $f_{\theta}^{\mathcal{T}\to\mathcal{L}}$. 

\subsection{Enforcing Stability}
In this subsection, we introduce a method that enforces the stability conditions in $f_{\theta}^{\mathcal{T}}$. 

\subsubsection{First condition $\left(f_{\theta}^{\mathcal{T}\to\mathcal{L}}=f^{\mathcal{L}}\right)$}
The first stability condition can be enforced by minimizing the distance between the states visited by the dynamical systems $f_{\theta}^{\mathcal{T}\to\mathcal{L}}$ and $f^{\mathcal{L}}$ when starting from the same initial condition. Hence, $\forall y^{\mathcal{T}}_{t}, y^{\mathcal{L}}_{t} \in \mathcal{L}$ a loss can be defined as $\ell_{\text{match}}=d(y_{t}^{\mathcal{T}}, y_{t}^{\mathcal{L}})$, where $d(\cdot, \cdot)$ is a distance function.

\subsubsection{Second condition $\left(\psi_{\theta}(x_{t})=y_{g} \Rightarrow x_{t}=x_{g}\right)$}
The second stability condition, however, can be more challenging to obtain, since we do not have a direct way of optimizing this in a DNN. Therefore, to achieve this, we introduce the following proposition:
\begin{proposition}[Surrogate stability conditions]
\label{prop:sep}
The second stability condition of Theo. \ref{theo:stable}, i.e., $\psi_{\theta}(x_{t})=y_{g} \Rightarrow x_{t}=x_{g}$, $\forall x_{t} \in \mathcal{T}$, is true if:
\begin{enumerate}[font=\itshape]
    \item $f_{\theta}^{\mathcal{T}\to\mathcal{L}}(x_{t})=f^{\mathcal{L}}(y_{t}^{\mathcal{T}})$, $\forall x_{t} \in \mathcal{T}$ (stability condition 1),
    \item $y^{\mathcal{T}}_{t-1} \neq y^{\mathcal{T}}_{t}$, $\forall x_{t} \in \mathcal{T}\setminus\{x_{g}\}$.
\end{enumerate}
\end{proposition}
\begin{proof}
If $y^\mathcal{T}_{t-1}=\psi(x_{t-1})=y_{g}$ the first condition implies that $y_{g}=y^{\mathcal{T}}_{t-1}=y^{\mathcal{T}}_{t}$, since $f^{\mathcal{L}}(y_{g})=0$. Consequently, given the second condition $y^{\mathcal{T}}_{t-1} \neq y^{\mathcal{T}}_{t}$, $\forall x_{t} \neq x_{g}$, this is only possible if $x_{t-1} = x_{g}$.
\end{proof}

In other words, Prop. \ref{prop:sep} indicates that if the first stability condition is true; then, we can obtain its second condition by enforcing $y^{\mathcal{T}}_{t-1} \neq y^{\mathcal{T}}_{t}$, $\forall x_{t} \neq x_{g}$. Notably, by enforcing this, the stability conditions are also enforced. Consequently, we refer to the conditions of Prop. \ref{prop:sep} as the \emph{surrogate stability conditions}.

Then, it only remains to define a loss $\ell_{\text{sep}}$ that enforces the second surrogate stability condition in $f_{\theta}^{\mathcal{T}}$. However, before doing so, note that the surrogate conditions aim to push some points together (i.e., $y^{\mathcal{T}}_{t}$ and $y^{\mathcal{L}}_{t}$) and separate others (i.e., $y^{\mathcal{T}}_{t-1}$ and $y^{\mathcal{T}}_{t}$). Hence, this problem overlaps with the Contrastive Learning (CL) and Deep Metric Learning literature \cite{kaya2019deep}.

\subsubsection{Contrastive Learning}
The problem of pushing some points together ($\ell_{\text{match}}$) and separating others ($\ell_{\text{sep}}$), is equivalent to the problem that the pairwise contrastive loss, from the CL literature, optimizes \cite{chopra2005learning}. This loss computes a cost that depends on positive and negative samples. Its objective is to reduce the distance between positive samples and separate negative samples beyond some margin value $m \in \mathbb{R}^{+}$.

In our problem, positive samples are defined as $y^{\mathcal{L}}_{t}$ and $y^{\mathcal{T}}_{t}$, and negative samples are defined as $y^{\mathcal{T}}_{t-1}$ and $y^{\mathcal{T}}_{t}$. The loss for positive samples is the same as $\ell_{\text{match}}$. Differently, for negative samples, this method separates points by minimizing $\ell_{\text{sep}}=\max(0, m - d(y^{\mathcal{T}}_{t-1}, y^{\mathcal{T}}_{t}))$, $\forall y^{\mathcal{T}}_{t-1}, y^{\mathcal{T}}_{t} \in \mathcal{L}$. If their distance is smaller than $m$, $m-d(y_{t-1}^{\mathcal{T}}, y_{t}^{\mathcal{T}}) > 0$, which is minimized until their distance is larger than $m$ and $m-d(y_{t-1}^{\mathcal{T}}, y_{t}^{\mathcal{T}}) < 0$. 

Commonly, the squared $l^{2}$-norm is used as the distance metric. Moreover, this loss is optimized along a trajectory starting at $t=1$, which is a state sampled randomly from the task space $\mathcal{T}$. Then, we define a contrastive loss for motion stability as
\begin{equation}
    \underbrace{\sum_{b=0}^{B^{s}-1}\sum_{t=1}^{H^{s}} \underbrace{||y^{\mathcal{L}}_{t, b}- y^{\mathcal{T}}_{t, b}||_{2}^{2}}_{\ell_{\text{match}}} + \underbrace{\max(0, m - ||y^{\mathcal{T}}_{t, b}- y^{\mathcal{T}}_{t-1, b}||_{2})^{2}}_{\ell_{\text{sep}}}}_{\ell_{\text{stable}}},
\label{eq:contrastive_loss}
\end{equation}
where $B^{s}, H^{s} \in \mathbb{N}^{+}$ are the batch size corresponding to the number of samples used at each training iteration of the DNN and $H^{s}$ is the trajectory length used for training, respectively. 

Note that \eqref{eq:contrastive_loss} does not take into account the fact that $\ell_{\text{sep}}$ should not be applied at $y_{g}$. However, in practice, it is very unlikely to sample $x_{g}$, so we do not deem it necessary to explicitly consider this case. Furthermore, the loss $\ell_{\text{match}}$ enforces $f^{\mathcal{T}}(y_{g})=0$, which also helps to keep $y^\mathcal{T}_{t-1}$ and $y^{\mathcal{T}}_{t}$ together when $y^\mathcal{T}_{t-1}=y_{g}$.

\subsubsection{Relaxing the problem}
We can make use of the CL literature to use other losses to solve this problem. More specifically, we study the triplet loss \cite{schroff2015facenet} as an alternative to the pairwise loss. We call this version CONDOR (relaxed), since, in this case, the positive samples $y^{\mathcal{T}}_{t}$ and $y^{\mathcal{L}}_{t}$ are pushed closer, but it is not a requirement for them to be the same. Hence, we aim to observe if learning a specific structure in $\mathcal{L}$ is enough to enforce stability in $f_{\theta}^{\mathcal{T}}$, even though \eqref{eq:contrastive_loss} is not solved exactly. This allows to compare different features between losses, such as generalization capabilities.

\subsection{Boundaries of the Dynamical System}\label{sec:boundaries}
We enforce the stability of a motion in the region $\mathcal{T}$ by randomly sampling points from it and minimizing \eqref{eq:contrastive_loss}. Since this property is learned by a DNN, stability cannot be ensured in regions of the state space where this loss is not minimized, i.e.,  outside of $\mathcal{T}$. Therefore, it is crucial to ensure that if a point belongs to $\mathcal{T}$, its evolution will not leave $\mathcal{T}$. In other words, $\mathcal{T}$ must be a \emph{positively invariant set} w.r.t. $f^{\mathcal{T}}_{\theta}$ \cite{lemme2014neural,khalil2015nonlinear}.

To address this, we design the dynamical system such that, by construction, is not allowed to leave $\mathcal{T}$. This can be easily achieved by projecting the transitions that leave $\mathcal{T}$ back to its boundary, i.e., if a point $x_{t} \in \mathcal{T}$ transitions to a point $x_{t+1} \notin \mathcal{T}$; then, it is projected to the boundary of $\mathcal{T}$. In this work, $\mathcal{T}$ is a hypercube; consequently, we apply an orthogonal projection by saturating/clipping the points that leave $\mathcal{T}$.

Note that this saturation is always applied, i.e., during the training and evaluation of the dynamical system. Hence, the stability conditions of Theo. \ref{theo:stable} are imposed on a system that evolves in the positively invariant set $\mathcal{T}$.

\subsection{Designing $f^{\mathcal{L}}$}
So far, we presented a method for coupling two dynamical systems such that they share stability properties; however, we assumed that $f^{\mathcal{L}}$ existed. In reality, we must design this function such that it is stable by construction. Although several options are possible, in this work, we define $f^{\mathcal{L}}$ as
\begin{equation}
\label{eq:DSy}
    \dot{y}_{t} = \alpha \odot (y_{g} - y_{t}),
\end{equation}
where $\alpha \in \mathbb{R}^{n}$ corresponds to the gains vector and $\odot$ to the element-wise/hadamard product. If $\alpha_{i} > 0$\footnote{This holds for the continuous-time case. However, we approximate the evolution of this system via the forward Euler integration method. Then, the system can be written for the discrete-time case as $y_{t+1}=Ay_{t}$, where $A = I + \text{diag}(-\alpha)\Delta t$, assuming $y_{g}=0$ without loss of generality. To ensure stability, the absolute value of the eigenvalues of $A$ must be less than one; then, $0 < \alpha_{i} < 2/ \Delta t$. }, this system monotonically converges to $y_{g}$ \cite{hunter2011introduction}, where $\alpha_{i}$ corresponds to the i-th element of $\alpha$. 

\subsubsection{Adaptive gains}\label{sec:adaptive_gains}
In the simplest case, $\alpha$ is a fixed, pre-defined, value; however, the performance of the learned mappings $\psi$ and $\phi$ is susceptible to the selected value of $\alpha$. Alternatively, to provide more flexibility to the framework, we propose to define $\alpha$ as a trainable function that depends on the current latent state $y_{t}$, i.e., $\alpha(y_{t})$. 
Then, the parameters of $\alpha$ can be optimized using the same losses employed to train $\psi$ and $\phi$, since it is connected to the rest of the network and the training error can be propagated through it. 

Note that this system is stable under the same condition for $\alpha$ as before, as shown in Appendix \ref{appendix:stability_adaptive}.

\subsection{Behavioral Cloning of Dynamical Systems}
Finally, we need to optimize an Imitation Learning loss $\ell_{\text{IL}}$ such that the learned dynamical system $f^{\mathcal{T}}_{\theta}$ follows the demonstrations of the desired motion. For simplicity, we opt to solve a Behavioral Cloning (BC) problem; however, in principle, any other IL approach can be used. As described in Section \ref{sec:preliminaries}, this can be achieved by minimizing the (forward) Kullback-Leibler divergence between the demonstration's trajectory distribution and the trajectory distribution induced by the learned dynamical system. Note that this problem formulation is equivalent to applying Maximum Likelihood Estimation (MLE) between these distributions \cite{bishop2006pattern}; hence, we can rewrite it as
\begin{equation}
    f^{\mathcal{T}}_{\theta} = \argmax_{f^{\mathcal{T}}_{\theta} \in \mathcal{F}} \mathbb{E}_{\tau \sim p^{*}(\tau)}\left[ \ln p_{\theta}(\tau)\right].
\label{eq:bc_1}
\end{equation}

If we note that $p_{\theta}(\tau)$ is a product of conditional transition distributions $p_{\theta}(x_{t+1}|x_{t})$, we can rewrite it as \text{$p_{\theta}(\tau)=\Pi^{T-1}_{t=0}p_{\theta}(x_{t+1}|x_{t})p(x_{0})$}, where $p(x_{0})$ is the initial state probability distribution. Replacing this in \eqref{eq:bc_1} and ignoring constants we obtain  
\begin{equation}
    f^{\mathcal{T}}_{\theta} = \argmax_{f^{\mathcal{T}}_{\theta} \in \mathcal{F}} \mathbb{E}_{\substack{x_{t+1} \sim p^{*}(x_{t+1}|x_{t}), \\ x_{t} \sim p^{t*}(x_{t})}}\left[\sum_{t=0}^{T-1} \ln p_{\theta}(x_{t+1}|x_{t})\right],
\label{eq:bc_2}
\end{equation}
where $p^{t*}(x_{t})$ is the probability distribution of states at time step $t$, and $p^{*}(x_{t+1}|x_{t})$ is the distribution of transitioning to state $x_{t+1}$ given that the system is in some state $x_{t}$. Both of these distributions are induced by the dynamical system $f^{*}$. 

In practice, however, we do not have an analytical representation of the distributions $p^{t*}(x_{t})$ and $p^{*}(x_{t+1}|x_{t})$. Therefore, the problem has to be estimated through empirical evaluations of this objective, which is achieved using the demonstrations present in the dataset $\mathcal{D}$. Then, we can solve this problem iteratively \cite{abramson2020imitating} by randomly sampling batches of $B^{i}$ trajectories from $\mathcal{D}$ at each iteration and maximizing
\begin{equation}
    f^{\mathcal{T}}_{\theta} = \argmax_{f^{\mathcal{T}}_{\theta} \in \mathcal{F}} \sum_{b=0}^{B^{i}-1}\sum_{t=0}^{T-1} \ln p_{\theta}(x^{*}_{t+1, b}|x_{t, b}),
\label{eq:bc_3}
\end{equation} 
where the subscript $b$ has been added to the states indicating their correspondence to the different trajectories of $B$.

To solve this problem, we can assume the transition distribution of the learning system to be a Gaussian with fixed covariance, and a mean corresponding to the \emph{forward Euler integration} \cite{legaard2021constructing} of $f^{\mathcal{T}}_{\theta}$ for the given state $x_{t,b}$, i.e., $x_{t+1,b}=x_{t,b} + f^{\mathcal{T}}_{\theta}(x_{t,b})\Delta t$, where $\Delta t$ corresponds to the time step size. Furthermore, the same Gaussian assumption is made for the demonstration's distribution $p^{*}(x_{t+1}|x_{t})$; however, since its transitions are obtained directly from the demonstrations, it is not necessary to integrate in this case. Then, \eqref{eq:bc_3} reduces to the Mean Squared Error (MSE) minimization between the mean of the demonstration's distribution $p^{*}(x_{t+1}|x_{t})$, and the mean of the learning distribution $p_{\theta}(x_{t+1}|x_{t})$ \cite{osa2018algorithmic}, i.e.,
\begin{equation}
    f^{\mathcal{T}}_{\theta} = \argmin_{f^{\mathcal{T}}_{\theta} \in \mathcal{F}} \sum_{b=0}^{B^{i}-1}\sum_{t=0}^{T-1}  ||x^{*}_{t+1,b} - \left(x_{t,b} + f^{\mathcal{T}}_{\theta}(x_{t,b})\Delta t\right)||_{2}^{2},
\label{eq:bc_4}
\end{equation}

In practice, however, if the trajectories of the demonstrations are too long, due to computation or complexity limitations, it might not be convenient to optimize this objective for the complete trajectories. Therefore, this problem can be simplified by allowing the initial conditions of the demonstration batches to be at any time step $t' \in \{0, ..., T-1\}$, and optimizing the problem for some time horizon $H^{i}\leq T$. Consequently, we get the loss $\ell_{\text{IL}}$ that we employ to solve the BC problem in this work:
\begin{equation}
    f^{\mathcal{T}}_{\theta} = \argmin_{f^{\mathcal{T}}_{\theta} \in \mathcal{F}} \underbrace{\sum_{b=0}^{B^{i}-1}\sum_{t=t'}^{H^{i}-1}  ||x^{*}_{t+1,b} - x_{t,b} - f^{\mathcal{T}}_{\theta}(x_{t,b})\Delta t||_{2}^{2}}_{\ell_{\text{IL}}}.
\label{eq:bc_5}
\end{equation}

\begin{figure}[t]
    \centering
    \includegraphics[width=\columnwidth]{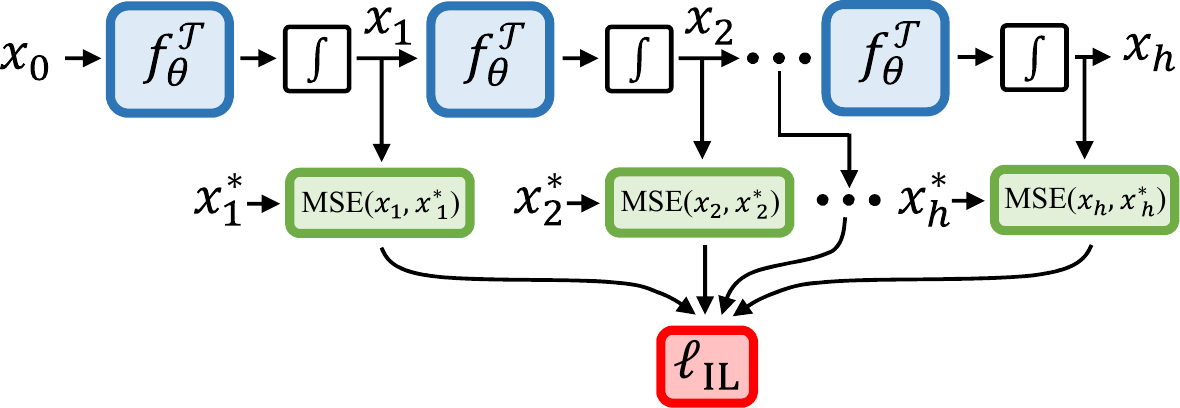}
    \caption{Multi-step IL loss for one sample when using backpropagation through time, where $h=H^{i}$.}
    \label{fig:multi_step}
\end{figure}
\subsection{Compounding errors and multi-step learning}
Commonly, \eqref{eq:bc_5} is solved as a single-step prediction problem (i.e., $H^{i}=1$) by computing only one transition from $x_{t',b}$ using $f^{\mathcal{T}}_{\theta}$ and comparing it against $x^{*}_{t'+1,b}$. Nevertheless, in practice, the learned dynamical system is applied recursively, i.e., assuming perfect tracking, every prediction is computed as a function of a previously computed output using the following equation:
\begin{equation}
    x_{h} = x_{h-1} + f^{\mathcal{T}}_{\theta}(... x_{1} + f^{\mathcal{T}}_{\theta}(x_{0} + f^{\mathcal{T}}_{\theta}(x_{0})\Delta t)\Delta t...)\Delta t,
\label{eq:recursive_prediction}
\end{equation}
where $h$ is the evolution horizon. Therefore, the prediction error of $f^{\mathcal{T}}_{\theta}$ compounds and grows multiplicatively by every new prediction \cite{lambert2022investigating,bontempi2012machine}. This makes the dynamical system diverge away from the states present in the demonstration's trajectories, requiring the system to make predictions in states that are not supported by the training data, which is known as the \emph{covariate shift} problem \cite{osa2018algorithmic}. Consequently, the prediction error grows even larger.

An important reason for this issue to occur is that the learned system is expected to act over multiple steps when it is only being trained for predicting single steps. To alleviate this problem, the dynamical system must be trained for predicting multiple steps, by setting $H^{i}>1$ and computing $x_{t,b}$ in \eqref{eq:bc_5} recursively, as shown in Fig. \ref{fig:multi_step}. In practice, however, the single-step loss is commonly employed, as the multi-step loss has been regarded as being challenging to optimize, even for short prediction horizons \cite{venkatraman2015improving}. Nevertheless, these challenges can be addressed with current DNN optimization techniques.

Consequently, we optimize the multi-step loss by noting that this can be achieved using \emph{backpropagation through time} \cite{werbos1990backpropagation,langford2009learning}, which has become popular and improved given its use in Recurrent Neural Networks (RNNs) \cite{graves2013generating}. Hence, its limitations such as \emph{exploding/vanishing gradients} or \emph{ill-conditioning} \cite{bengio1994learning} have been alleviated. Furthermore, specifically for this case, we can observe that every forward integration step in \eqref{eq:recursive_prediction} can be interpreted as one group of layers inside a larger DNN that computes $x_{h}$. Then, each one of these groups has the same structure as the \emph{residual blocks} in \emph{ResNet} \cite{he2016deep}, which have also shown to be beneficial for alleviating vanishing/exploding gradients issues \cite{veit2016residual}.



\section{Simulated experiments}\label{sec:simulated_experiments}
In this section, we employ datasets of human handwriting motions to validate our method. Although these datasets contain human demonstrations, our evaluation of the learned motions is \emph{simulated}, since no real system is involved in this process.  This can be better understood with Fig. \ref{fig:control_strategies}, where we show two different control strategies that can be employed with CONDOR. More specifically, Fig. \ref{fig:offline_control} presents an offline control strategy where a trajectory is computed and stored in a buffer by applying CONDOR recursively. Afterwards, this trajectory is tracked by a low-level controller. In our evaluation, however, we ignore the low-level controller part and evaluate CONDOR using only the trajectory provided by the buffer, i.e., we assume that the trajectory is tracked perfectly. Despite this assumption, this methodology with this dataset has been extensively used in the literature, since it allows to test if the learning method generates adequate state transition requests \cite{khansari2011learning,rana2020euclideanizing,urain2020imitationflow}.

The DNN architecture and hyperparameter optimization process of the models used in this section are described in Appendices \ref{appendix:neural_network} and \ref{appendix:hyerparameter_optimization}, respectively.

\begin{figure}[t]
\centering
\subfloat[Online control. At every time step, $f^{\mathcal{T}}_{\theta}$ receives $x_{t}$ from the sensors of the robot, and its output is fed to a low-level controller that tracks it. ]{\includegraphics[width=0.4\linewidth,valign=t]{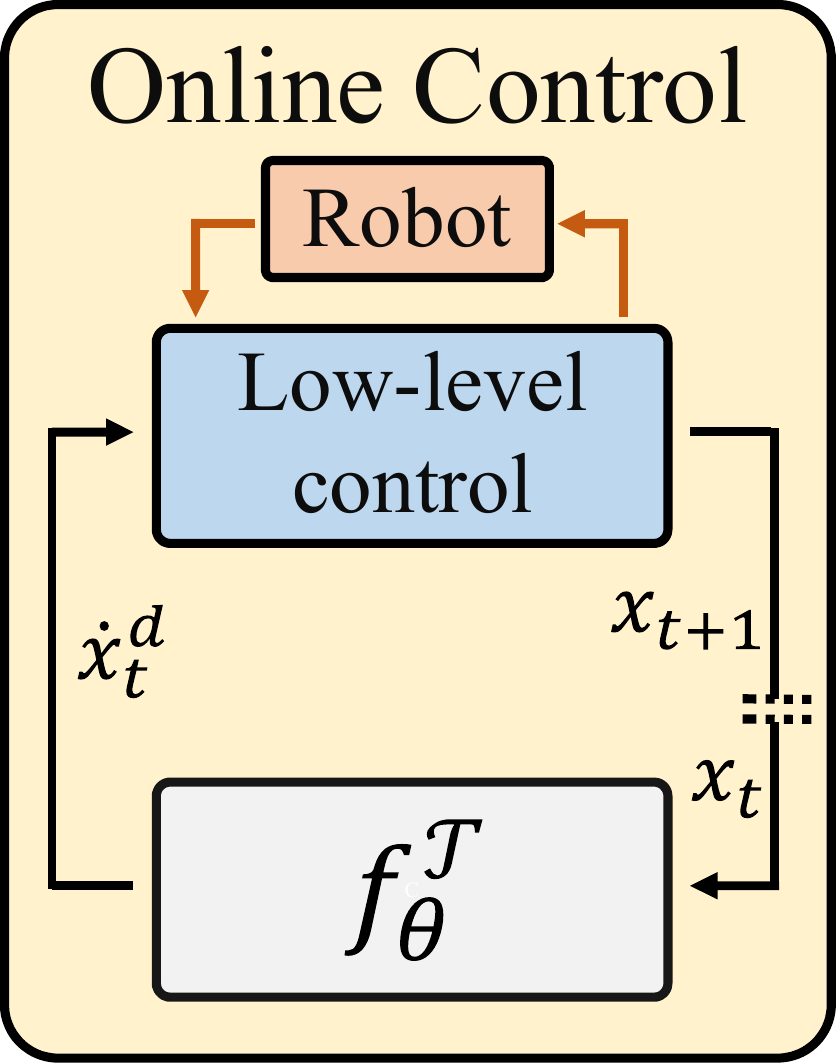}%
\vphantom{\includegraphics[width=0.5\linewidth,valign=t]{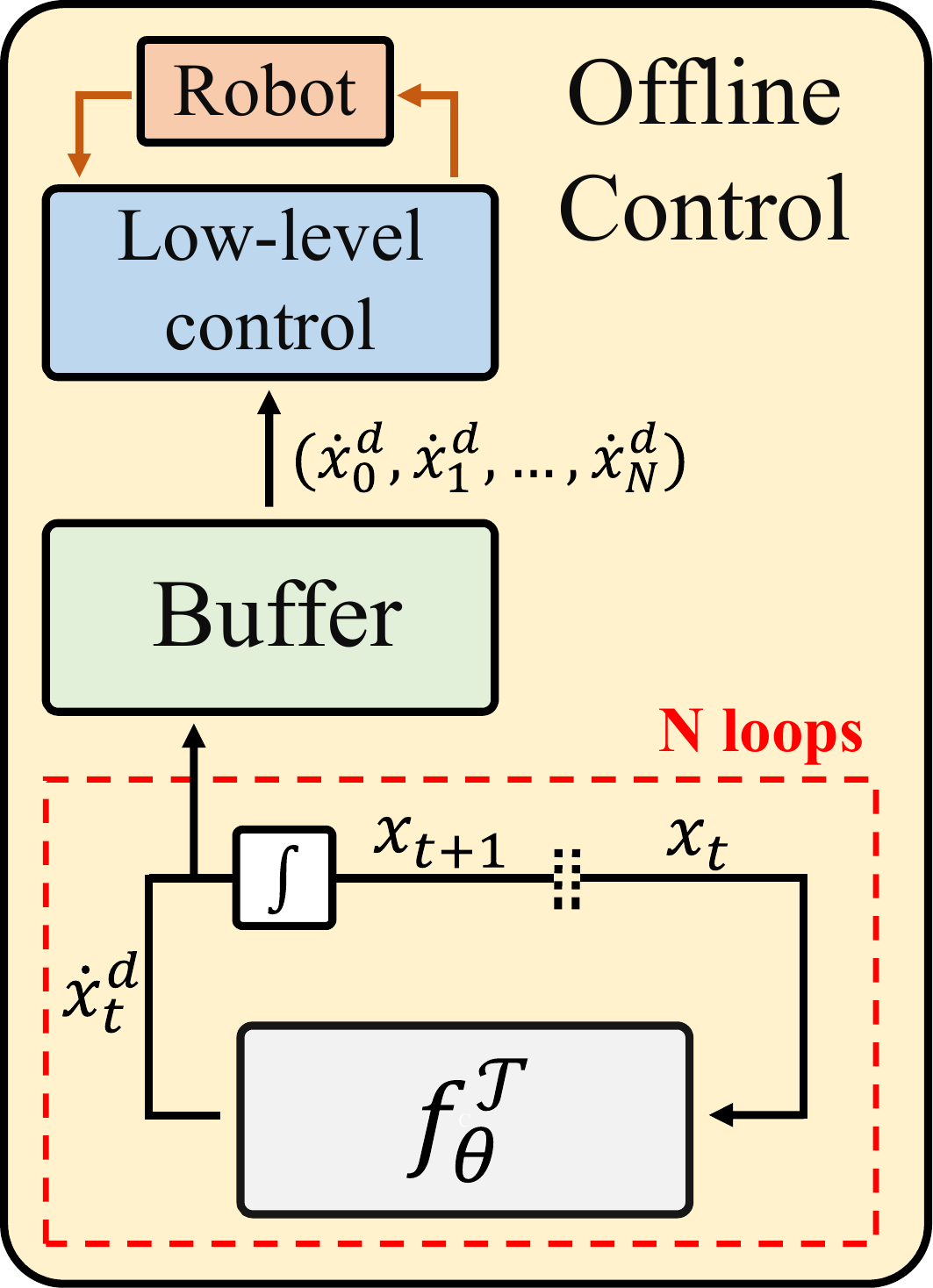}}%
\label{fig:online_control}%
}
\quad
\subfloat[Offline control. $f^{\mathcal{T}}_{\theta}$ is applied recursively N times to create, offline, a trajectory that is stored in a buffer. Afterward, the complete trajectory is tracked with a low-level controller.]{\includegraphics[width=0.5\linewidth,valign=t]{figures/offline_control.pdf}%
\label{fig:offline_control}%
}
\caption{Control strategies that can be used with CONDOR.}
\label{fig:control_strategies}
\end{figure}

\begin{figure*}[t]
    \centering
    \includegraphics[width=0.9\linewidth]{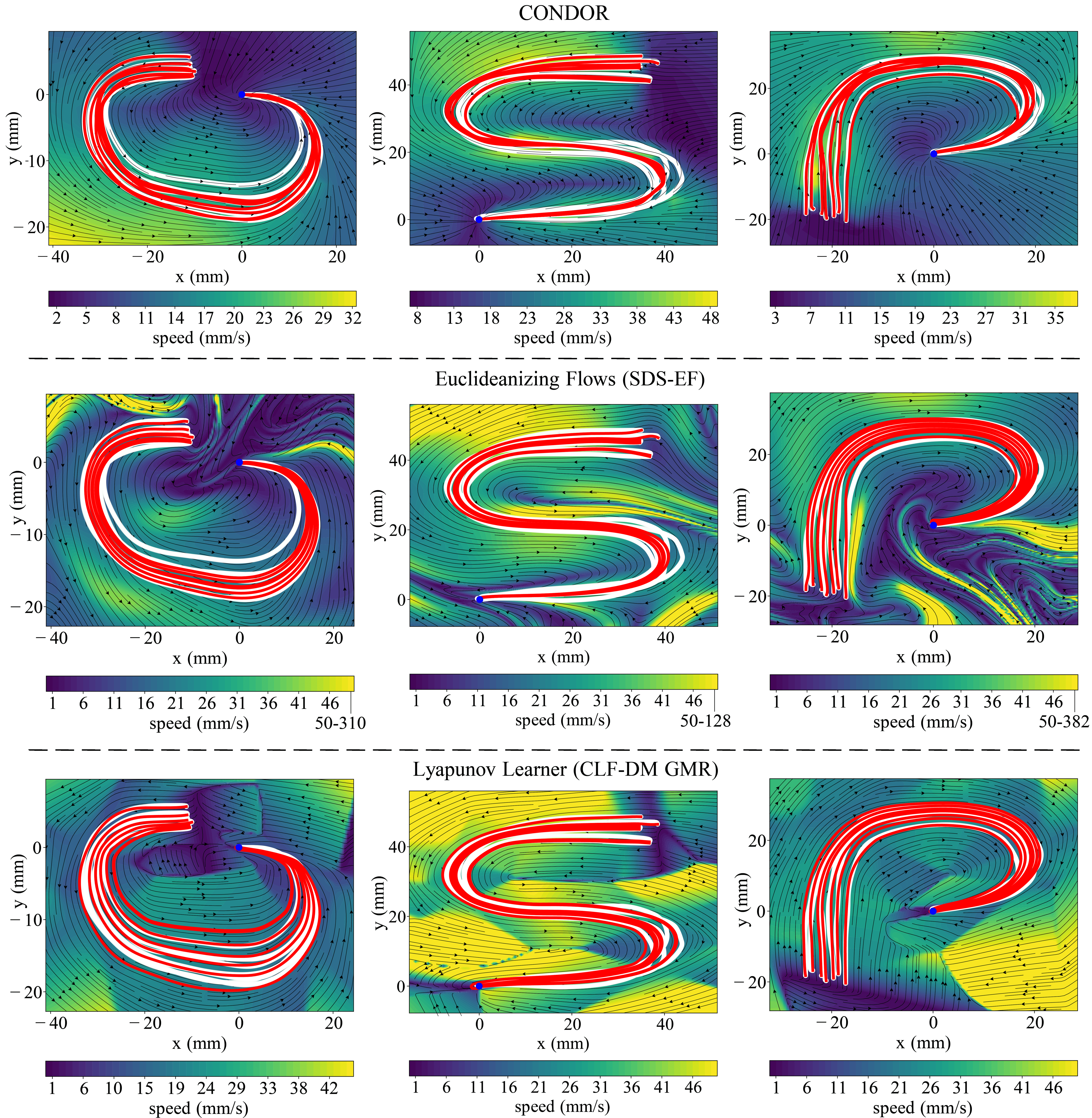}
    \caption{Examples of LASA dataset motions learned using CONDOR, SDS-ES, and CLF-DS (GMR). White curves represent the demonstrations. Red curves represent the executed motions by the learned model when starting from the same initial positions as the demonstrations. The arrows indicate the vector field of the learned dynamical system (velocity outputs for every position). In SDS-ES, every speed greater than 50 mm/s is saturated to this value.}
    \label{fig:LASA_examples}
\end{figure*}
%

\begin{figure}[t]
    \centering
    \includegraphics[width=0.8\columnwidth]{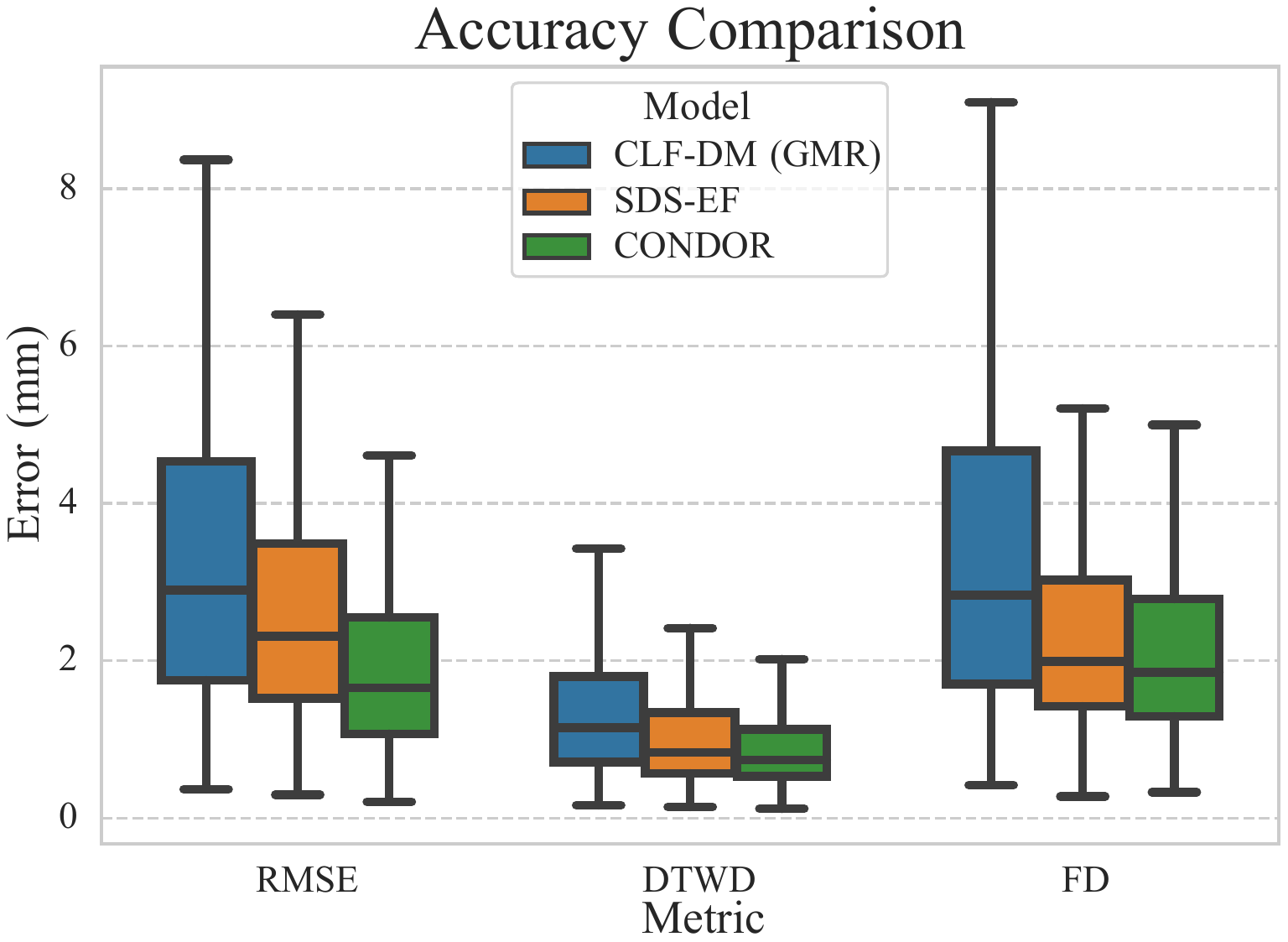}
    \caption{Accuracy comparison of CONDOR against state-of-the-art methods. Each box plot summarizes performance over the 30 motions of the LASA dataset.}
    \label{fig:boxplot_art_state_comparison}
\end{figure}

\subsection{LASA dataset validation: first-order 2-dimensional motions}\label{sec:lasa_validation}
We validate our method using the LASA dataset\footnote{\url{https://cs.stanford.edu/people/khansari/download.html}}, which comprises 30 human handwriting motions. Each motion, captured with a tablet PC, includes 7 demonstrations of a desired trajectory from different initial positions. The state is represented as 2-dimensional positions, and the learned systems are of first order, i.e., the output of $f_{\theta}^{\mathcal{T}}$ is a desired velocity. Although the demonstrations may have local intersections due to human inaccuracies, the shapes contained in this dataset can be well represented using $\text{first-order}$ dynamical systems, which cannot represent intersections. Consequently, we employ the LASA dataset to evaluate motions modeled as first-order dynamical systems, which is the same approach that was taken by the paper that introduced this dataset \cite{khansari2011learning}.

Fig. \ref{fig:LASA_examples} shows three examples of dynamical systems learned with CONDOR. These motions share three features:
\begin{enumerate}
    \item \textbf{Adequate generalization:} motions generated in regions with no demonstrations smoothly generalize the behavior presented in the demonstrations.
    \item \textbf{Accuracy:} the learned models accurately reproduce the demonstrations.
    \item \textbf{Stability:} the vector fields suggest that, independently of the initial conditions, every motion reaches the goal.
\end{enumerate}

In the following subsections, we provide further details regarding each one of these points. Moreover, we compare CONDOR\footnote{CONDOR code repository: \url{https://github.com/rperezdattari/Stable-Motion-Primitives-via-Imitation-and-Contrastive-Learning}} with two other state-of-the-art methods for stable motion generation: 1) Control Lyapunov Function-based Dynamic Movements (CLF-DM) using Gaussian Mixture Regression (GMR)\footnote{CLF-DM code repository: \url{https://github.com/rperezdattari/Learning-Stable-Motions-with-Lyapunov-Functions}} \cite{khansari2014learning}, and 2) Stable Dynamical System learning using Euclideanizing Flows (SDS-EF)\footnote{SDS-EF code repository: \url{https://github.com/mrana6/euclideanizing_flows}} \cite{rana2020euclideanizing}. CLF-DM (GMR) learns a dynamical system using a GMR and corrects its behavior whenever it is not stable according to a learned Lyapunov function. SDS-EF is a diffeomorphism shaping method, as introduced in Section \ref{sec:connection_diffeo}.

\subsubsection{Generalization}
Fig. \ref{fig:LASA_examples} also depicts the performance of CLF-DM and SDS-EF on three motions. Here, we observe that even though the stability of these methods is guaranteed, unlike CONDOR, the behavior that they present in regions without demonstrations might not always be desired. 

In the case of SDS-EF, unpredictable motions can be generated\footnote{These results are not completely consistent with the ones reported in \cite{rana2020euclideanizing}, since we removed additional preprocessing (smoothing and subsampling) to compare every method under the same conditions.}  (e.g., bottom image, bottom-right quadrant), which, furthermore, can reach very high speeds (e.g., 382 mm/s, while the demonstrations exhibit maximum speeds of around 40 mm/s). Note, however, that this issue can be alleviated by optimizing SDS-EF for a shorter period of time, but this also makes it less accurate. 
Such unpredictability and high speeds can be a limitation in real-world scenarios. For instance, when humans interact with robots and must feel safe around them, or due to practical limitations, e.g., it is unfeasible to track the requested motions with a low-level controller.

Differently, in the case of CLF-DM, nonsmooth transitions are present in some regions of the state space due to the corrections applied by the Lyapunov function. This can also be a limitation, since robotic systems commonly avoid nonsmooth trajectories to minimize the risk of damage \cite{ravichandar2020recent}.

Lastly, Fig. \ref{fig:LASA_examples} evidences that, in real-world scenarios, CLF-DM and SDS-EF are susceptible to making robots leave their workspaces. These methods do not constrain their trajectories to reside inside a specific space, they only guarantee that, eventually, these will converge to the goal. In practice, however, the learned trajectories might need to leave a robot's workspace to reach the goal. Then, in Fig. \ref{fig:LASA_examples}, if we assume that the observed regions are a robot's workspace, the vector fields of CLF-DM and SDS-EF indicate that some motions depart from it. In contrast, in CONDOR the workspace is a positively invariant set w.r.t. the learned dynamical system (see Section \ref{sec:boundaries}); consequently, motions stay inside it.  

\subsubsection{Accuracy}
\label{sec:accuracy_description}
Fig. \ref{fig:LASA_examples} indicates that every method is able to accurately reproduce the demonstrations. However, CLF-DM is clearly less accurate than CONDOR and SDS-EF. For instance, in the bottom-left image, the inner red trajectory drifts away from the demonstrations, coming back to them at the end of the motion due to its stability properties.

Quantitatively speaking, we can employ different metrics to evaluate the accuracy of the learned trajectories (see Fig. \ref{fig:boxplot_art_state_comparison}). Commonly, a distance between two trajectories is minimized; one trajectory corresponds to a demonstration, and the other corresponds to the one generated by the learned dynamical system when starting from the same initial condition as the demonstration. However, different distances between trajectories can be computed depending on the features that we aim to evaluate from the trajectories. To have a more complete view of the accuracy performance of our method, we compare CONDOR, CLF-DM (GMR)\footnote{Each GMR consisted of 10 Gaussians and each Lyapunov function was estimated using 3 asymmetric quadratic functions.} and SDS-EF\footnote{Results were extracted from \cite{rana2020euclideanizing}.} under three distance metrics: 1) Root Mean Squared Error (RMSE), 2) Dynamic Time Warping Distance (DTWD) \cite{muller2007dynamic}, and 3) Frechet Distance (FD) \cite{eiter1994computing}. 

We can observe that CONDOR clearly achieves better results against CLF-DM (GMR) under every metric, while a smaller gap, yet superior, is achieved against SDS-EF. 

\subsubsection{Stability}
Lastly, we quantitatively study the stability properties of CONDOR. As mentioned in Section \ref{sec:method}, the stability of the motions it learns depends on the optimization problem being properly minimized. Therefore, we need to empirically test this after the optimization process finishes.

To achieve this, we integrate the dynamical system for $L$ time steps, starting from $P$ initial states, and check if the system converges to the goal (i.e., fixed-point iteration, where the fixed point corresponds to the goal). The larger the $P$, the more accurate the results we obtain. If $L$ is large enough, the system should converge to the goal after $L$ steps. Hence, by computing the distance between the last visited state and the goal, and checking that it is below some predefined threshold $\epsilon$, it is possible to evaluate if a trajectory is successful or not (i.e., if it converges to the goal).

We evaluated CONDOR using all of the motions present in the LASA dataset with $L=2000$ and $P=1225$ with $\epsilon=1$mm, and observed that $100\%$ of the trajectories reached the goal. Hence, CONDOR is able to successfully learn stable motions.

\begin{figure}[t]
    \centering
    \includegraphics[width=\columnwidth]{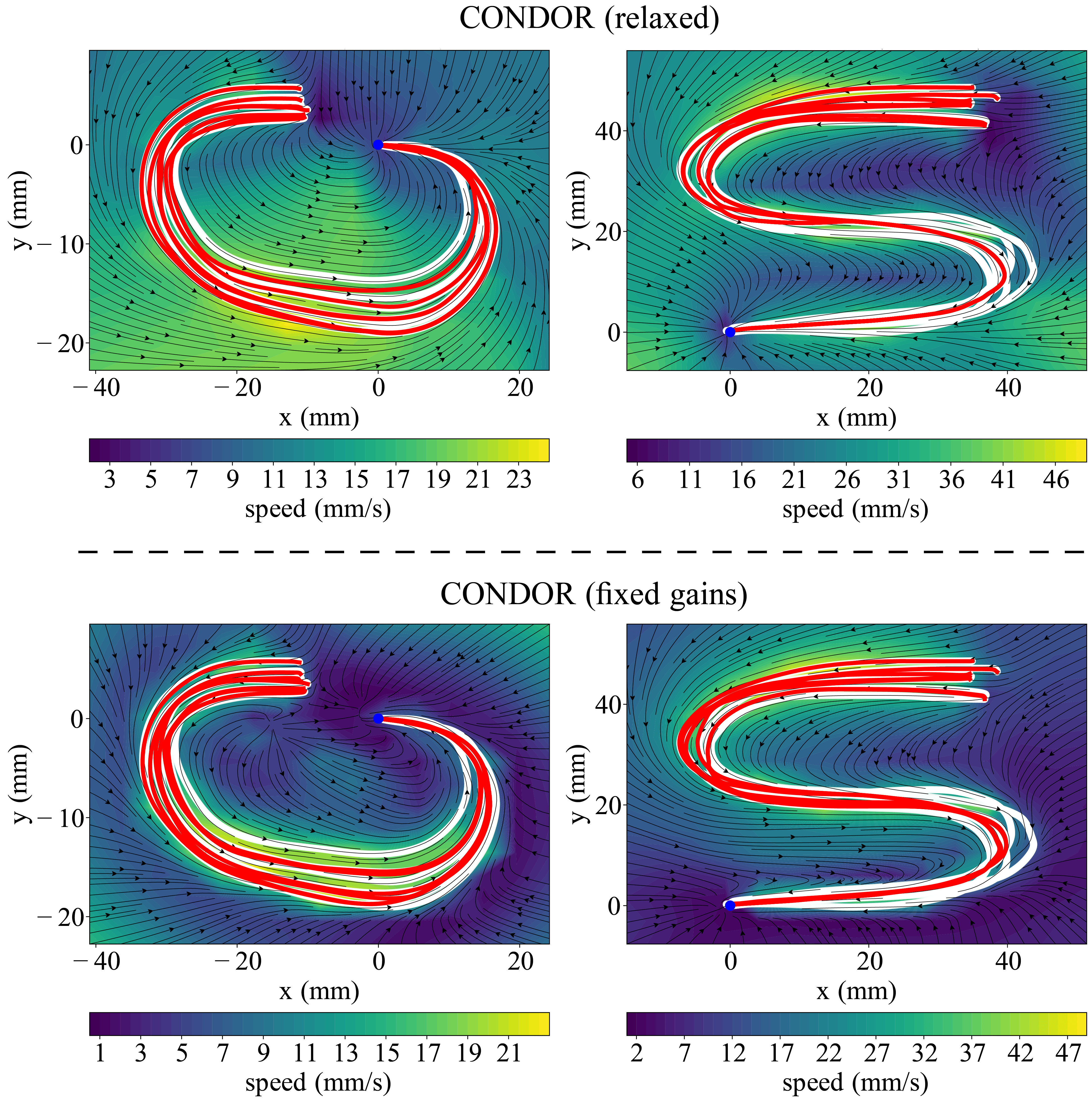}
    \caption{Examples of LASA dataset motions learned using two variations of CONDOR: 1) relaxed and 2) fixed gains.}
    \label{fig:quali_relaxed_fixed}
\end{figure}

\begin{figure}[t]
    \centering
    \includegraphics[width=0.8\columnwidth]{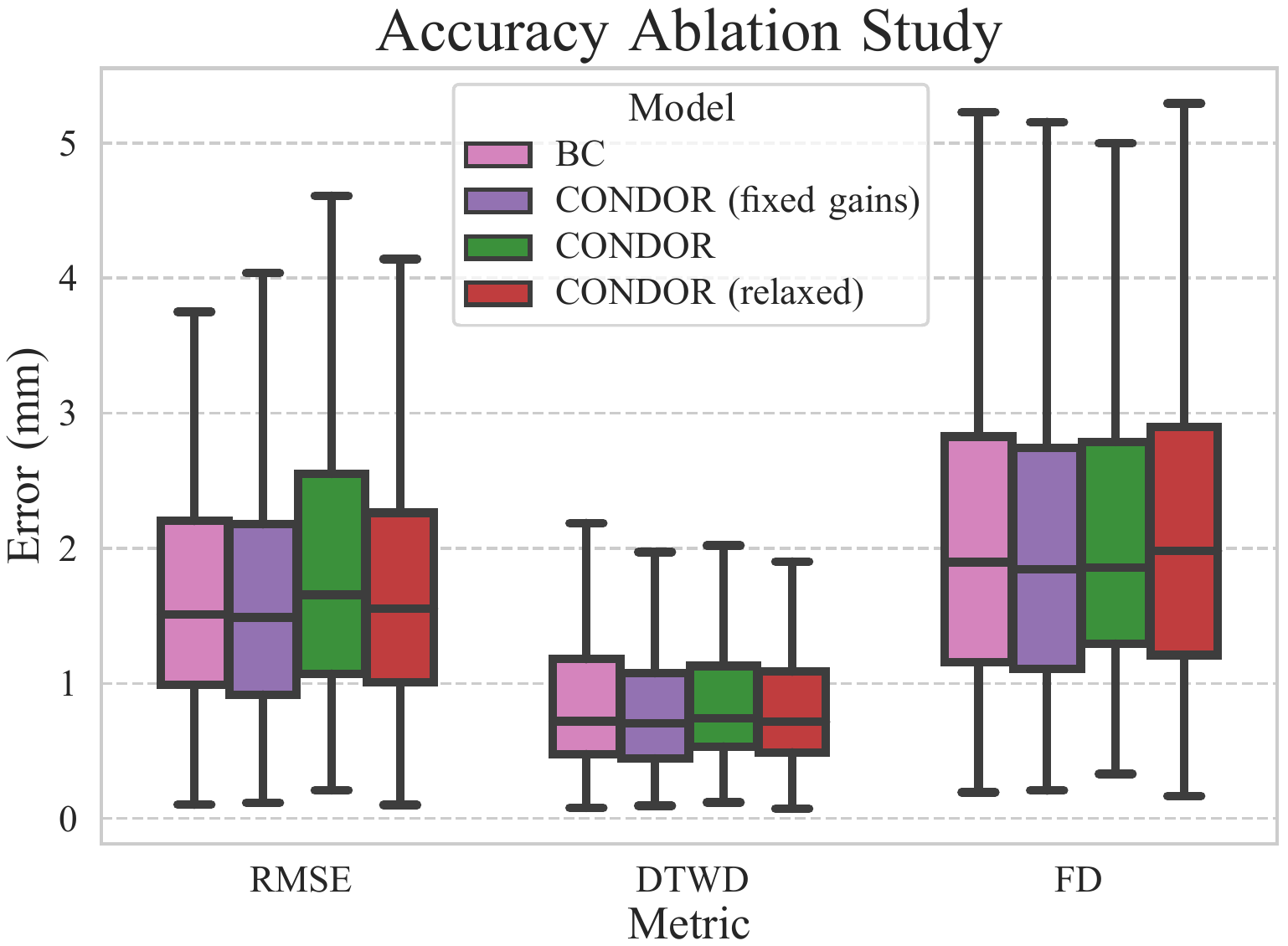}
    \caption{Accuracy comparison of different variations of CONDOR. Each box plot summarizes performance over the 30 motions of the LASA dataset.}
    \label{fig:accuracy_ablation}
\end{figure}

\subsection{Ablation Study}\label{ref:lasa_ablation}
To better understand CONDOR and the relevance of its different parts, we perform an ablation study where we compare four variations of this method:
\begin{enumerate}
    \item \textbf{CONDOR:} the base method studied in Section \ref{sec:lasa_validation}.
    \item \textbf{CONDOR (relaxed):} the contrastive loss for stability is replaced with the triplet loss. This variation is presented to observe the importance of minimizing the exact loss presented in \eqref{eq:contrastive_loss} or whether it is enough to enforce this type of structure in the latent space of the NN to obtain stable motions.
    \item \textbf{CONDOR (fixed gains):} as explained in Section \ref{sec:adaptive_gains}, having adaptive gains in the latent dynamical system described in \eqref{eq:DSy} should help obtaining more flexible motions. Therefore, this model, with fixed gains, is studied to observe the relevance and effect of using adaptive gains. 
    \item \textbf{Behavioral Cloning (BC):} the stability loss is removed and only the BC loss is employed to learn motions. This model is used to study the effect that the stability loss can have on the accuracy of the learned motions. To observe the behavior of BC, we refer the reader to Fig. \ref{fig:bc_example_DS}.
\end{enumerate}

Fig. \ref{fig:quali_relaxed_fixed} showcases examples of motions learned with both CONDOR (relaxed) and CONDOR (fixed gains). In both cases, the motions display accuracy and stability. However, these models differ in their generalization. CONDOR (relaxed) has a generalization behavior similar to the one of CONDOR shown in Fig. \ref{fig:LASA_examples}. In these cases, the regions of the state space without demonstrations exhibit a trend that resembles the one observed in the demonstrations. In contrast, the generalization of CONDOR (fixed gains) does not follow this trend as closely. For example, within certain regions, the velocity of the motions decreases, and as they approach the demonstrations, their direction becomes nearly orthogonal, indicating a discrepancy between the generalized behavior and the pattern presented in the demonstrations.

\begin{figure}[t]
    \centering
    \includegraphics[width=\columnwidth]{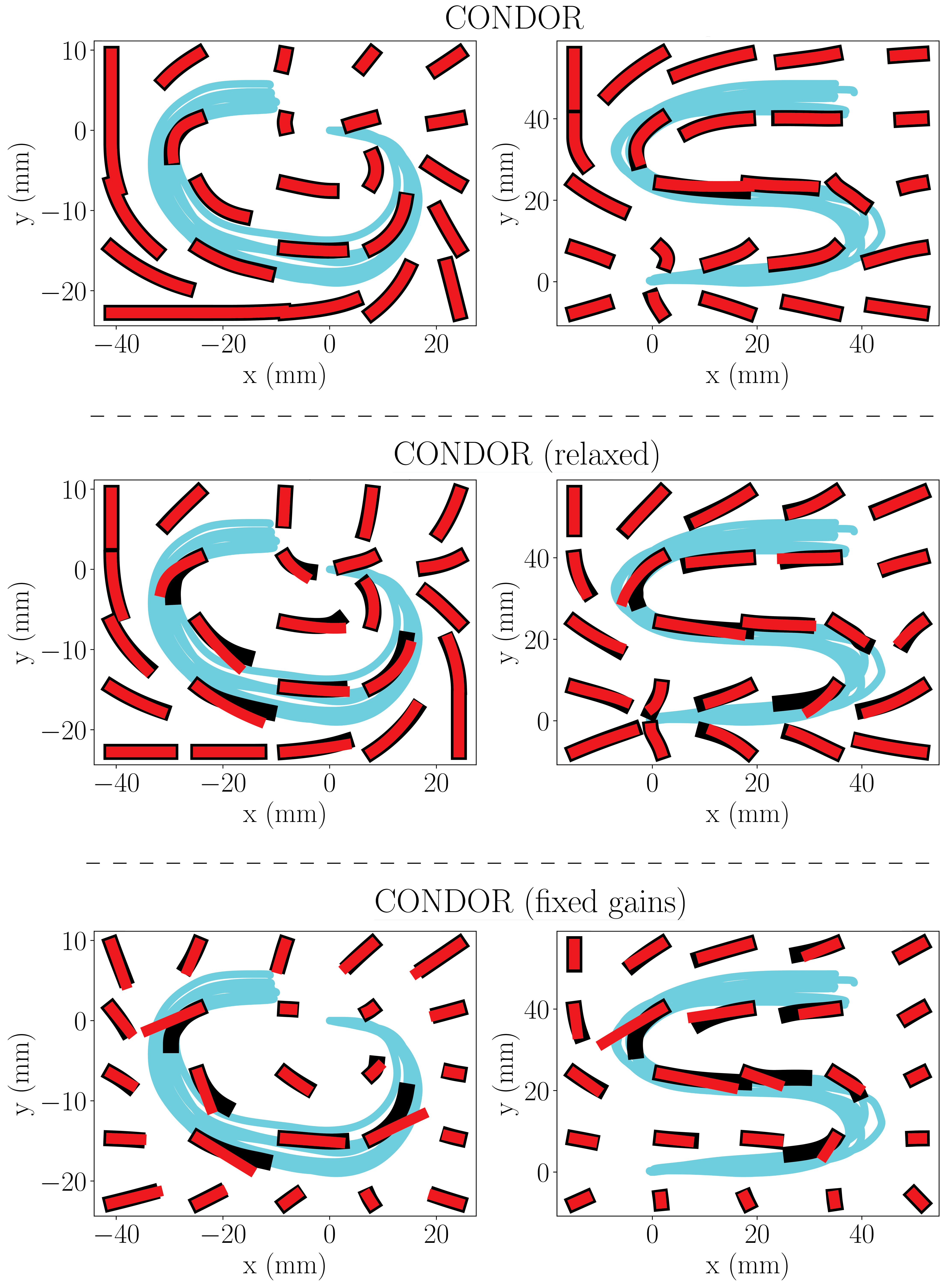}
    \caption{Dynamical systems mismatch comparison. Blue curves represent the demonstrations, black curves represent trajectories generated using $f_{\theta}^{\mathcal{T}}$ of the dynamical system, and red curves represent trajectories generated using $f^{\mathcal{L}}$ of the dynamical system. The black and red trajectories were obtained by integrating the dynamical system through 80 time steps.}
    \label{fig:diffeo_mismatch_quali}
\end{figure}
\subsubsection{Accuracy}
In this subsection, we compare the accuracy of the different variations of CONDOR (see Fig. \ref{fig:accuracy_ablation}). Intuitively, BC should perform better than any variation of CONDOR, since it only optimizes the BC loss. In contrast, the other variations also optimize the stability loss, which could harm/limit the minimization of the BC loss. Consequently, BC is the lower bound for the accuracy performance, i.e., best case scenario, a variation of CONDOR performs as well as BC does.

In Fig. \ref{fig:accuracy_ablation}, we observe that the accuracy performance of all of the variations of CONDOR is very similar, including the BC case. This result shows that CONDOR, and its variations, is able to effectively minimize the BC and stability loss together without harming the accuracy performance of the learned motions.

\subsubsection{Stability}
Table \ref{tab:stability_1} shows the results of stability tests on the presented methods. We can observe that when stability is not enforced (i.e., Behavioral Cloning), the percentage of unsuccessful trajectories is significant, being larger than one-third of the total amount of trajectories. In contrast, when stability is enforced using adaptive gains, the system achieves perfect performance (i.e., every trajectory reaches the goal), as it is observed with the results of CONDOR and CONDOR (relaxed). Interestingly, this result also shows that the relaxed variation of CONDOR can be employed for achieving stable motions without having a loss in performance. In contrast, when fixed gains are employed, although the percentage of unsuccessful trajectories is very low ($<0.01\%$), the performance degrades. This result suggests that the stability loss is not being as effectively minimized as when the gains are adaptive.
\begin{table}[t]
\centering
\caption{Percentage of unsuccessful trajectories over the LASA dataset ($L=2000$, $P=1225$, $\epsilon=1$mm).}
\label{tab:stability_1}
\begin{tabular}{llll}
\hline
\thead{Behavioral \\ Cloning} & \thead{CONDOR \\ (fixed gains)} & \thead{CONDOR \\ (relaxed)}  & \thead{CONDOR \\ \hphantom{h}} \\ \hline
36.4653\%          & 0.0054\%     & \textbf{0.0000}\%         & \textbf{0.0000}\% \\ \hline
\end{tabular}
\end{table}

\subsubsection{Dynamical Systems Mismatch}\label{sec:diffeo_mismatch}
So far, we observed that every variation of CONDOR that minimizes the stabilization loss is able to learn accurate and stable motions. The case of CONDOR (fixed gains) showed a slightly worse stability performance and poorer generalization capabilities than CONDOR and CONDOR (relaxed). However, CONDOR has not shown to be clearly superior to its variations, especially in the CONDOR (relaxed) case. 

Since CONDOR (relaxed) approximates the loss that minimizes the distance between $y_{t}^{\mathcal{L}}$ and $y_{t}^{\mathcal{T}}$, the trajectories that it obtains with $f^{\mathcal{L}}$ and $f_{\theta}^{\mathcal{T}}$ in the latent space should diverge faster than the ones generated with CONDOR. To investigate this idea, we evaluate the optimization of this loss by separately simulating $f^{\mathcal{T}}_{\theta}$ and $f^{\mathcal{L}}$ when starting from the same initial conditions. If the stabilization loss is perfectly minimized, these simulations should yield the same trajectories when mapping the evolution of $f^{\mathcal{L}}$ to task space\footnote{Trajectories from $f^{\mathcal{L}}$ with known initial conditions in $\mathcal{T}$ (hence, $y^{\mathcal{L}}_{0}=\psi_{\theta}(x_{0})$), can be mapped to $\mathcal{T}$ by recursively applying $x_{t+1}=x_{t} + \phi_{\theta}(y^{\mathcal{L}}_{t})\Delta t$.}; otherwise, they should diverge from each other. 

Fig. \ref{fig:diffeo_mismatch_quali} presents motions learned using the different variations of CONDOR and shows motions generated in task space when following $f_{\theta}^{\mathcal{T}}$ and $f^{\mathcal{L}}$. We can observe that CONDOR performs well in the complete state space, where, for most trajectories, it is not possible to detect a difference between the results obtained using $f^{\mathcal{L}}$ and $f_{\theta}^{\mathcal{T}}$. In contrast, we can observe that, for the other cases, trajectories diverge more pronouncedly. Interestingly, the divergent trajectories seem to overlap with the regions where demonstrations are provided. This suggests the stabilization loss is not properly minimized in this region, indicating that these variations of CONDOR struggle to find good solutions in the regions of the state space where the imitation and stabilization losses are optimized together, i.e., in the demonstrations. Finally, it is also possible to observe that CONDOR (relaxed) obtains trajectories that are slightly more similar to the ones obtained with CONDOR (fixed gains), although it is not conclusive. 

\begin{figure}[t]
    \centering
    \includegraphics[width=0.8\columnwidth]{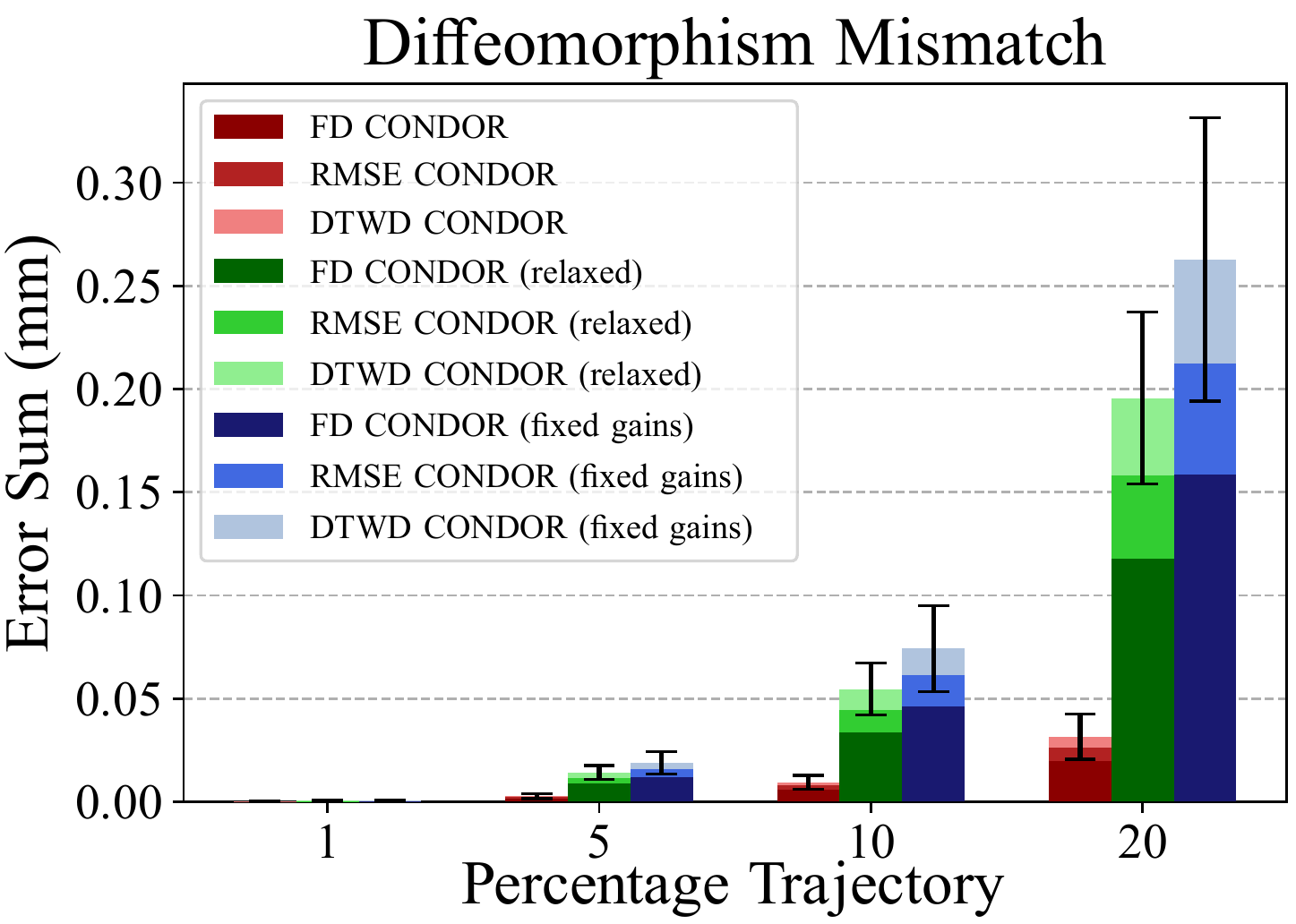}
    \caption{Trajectory mismatch error. Results are presented as a function of the length of a trajectory with respect to the complete length of the demonstrated trajectories, corresponding to 1000 transitions. 100 trajectories in the task and latent spaces were simulated, whose initial positions were uniformly distributed in the motion's state space. The results show the mean and half of the standard deviation of the error computed with these simulations.}
    \label{fig:diffeo_mismatch_quanti}
\end{figure}
Quantitatively, we can analyze this trajectory difference by computing the accumulated error between the trajectories generated using both dynamical systems. Fig. \ref{fig:diffeo_mismatch_quanti} shows this error as a function of the trajectory length. As expected, this error grows for the dynamical systems as a function of their length. However, CONDOR obtains a significantly lower error than its variations. Furthermore, Fig. \ref{fig:diffeo_mismatch_quanti} clearly shows that CONDOR (relaxed) outperforms CONDOR (fixed gains). Finally, since this accumulated error is a consequence of how well the stability loss is minimized, these results might explain why the stability performance of CONDOR (fixed gains) is not perfect, i.e., this variation is not able to successfully minimize the stability loss in the complete state space.

As a final remark, we can note that generating trajectories in the latent space and then mapping them to task space can have other applications, such as predicting future states efficiently and employing them, for instance, in Model Predictive Control frameworks. It is considerably faster to generate trajectories in the latent space than using the complete DNN architecture to compute them in task space, since the number of parameters and layers required to do so is smaller. Then, once the trajectory is generated in the latent space, it can be mapped to task space as one batch in one forward pass.

\subsection{LAIR dataset validation: second-order 2-dimensional motions}\label{sec:lair}
\begin{figure}[t]
    \centering
    \includegraphics[width=\columnwidth]{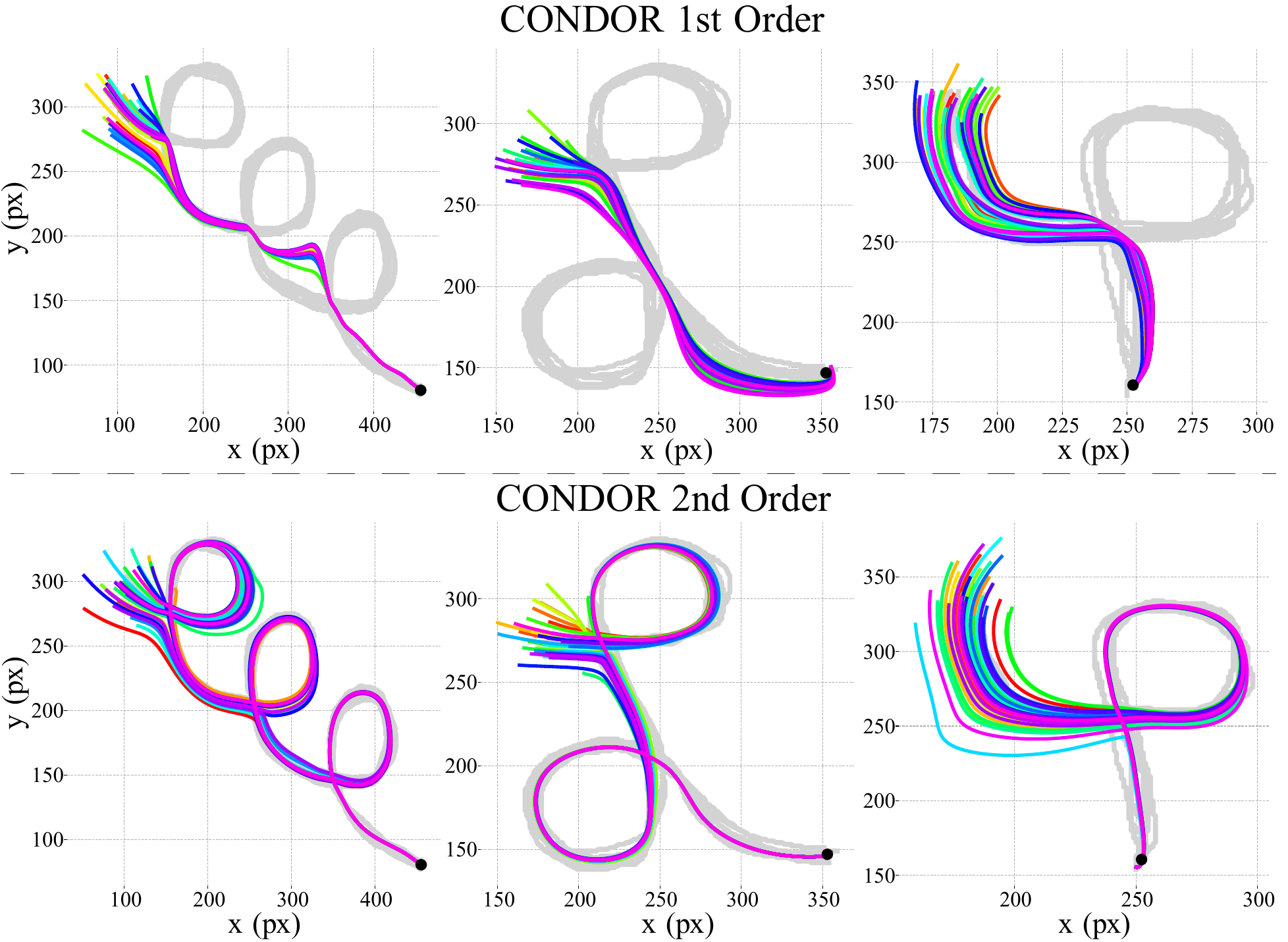}
    \caption{Motions modeled using CONDOR with first order and second order systems. The shapes in grey correspond to the demonstrations. The colored curves correspond to different instances of trajectories generated when starting from different initial states. Every trajectory was initialized with zero velocity and the initial positions were obtained by sampling from a Gaussian distribution around the initial positions observed in the demonstrations. 36 trajectories were sampled per plot.}
    \label{fig:second_order_quali}
\end{figure}
In this section, we introduce the \emph{LAIR handwriting dataset}. The objective is to test the accuracy and stability performance of the proposed method for second-order motions, where the state comprises both position and velocity. This dataset contains 10 human handwriting motions collected using a mouse interface on a PC. The state here is 4-dimensional, encompassing a 2-dimensional position and velocity, and the output of $f^{\mathcal{T}}_{\theta}$ is the desired acceleration. The dataset's shapes present several position intersections that have been designed to require, at least, second-order systems to model them. This dataset is employed to test the scalability of the proposed method in terms of the order of the motion.

Unlike the LASA dataset, the LAIR dataset contains raw demonstrations without any type of postprocessing. Hence, the ending points of the demonstrated trajectories might not always coincide exactly. To account for this, the goal of a motion is computed by taking the mean between these ending points.

\subsubsection{Accuracy}
Fig. \ref{fig:second_order_quali} shows three examples of motions of the LAIR dataset. These motions can only be modeled using a dynamical system of, at least, second order. First-order systems only employ position information to generate a trajectory; hence, visiting the same position two times will generate an ambiguity for the learning algorithm. This makes the learned system collapse to a solution that lies in between the multiple demonstrated options. Therefore, we observe that first-order systems with CONDOR are not able to appropriately model the shown motions.

In contrast, we observe that second-order systems are able to appropriately capture the dynamics of the demonstrated motions and execute them as they were intended. However, some trajectories (especially those coming from the tail of the initial-state sampling distribution) do not go through the first intersection, since they start from a position that, given its distance from the initial states of the demonstrations, directly follows the trend of the motion after this intersection. If this is a limitation for a specific application, providing demonstrations in those regions would make the system behave as expected. Finally, another interesting feature of these motions, is that the different trajectories, eventually, seem to collapse to the same position, overlapping with each other. This comes as an artifact when incorporating the stability loss, where the systems find these solutions to ensure stability.

Quantitatively, the same conclusions can be drawn when observing Fig. \ref{fig:box_plot_second_order}. This figure presents the results of the accuracy of both CONDOR variations under the same metrics employed in Section \ref{sec:accuracy_description}. As expected, the second-order systems outperform the first-order systems by a large margin.
\begin{figure}[t]
    \centering
    \includegraphics[width=0.8\columnwidth]{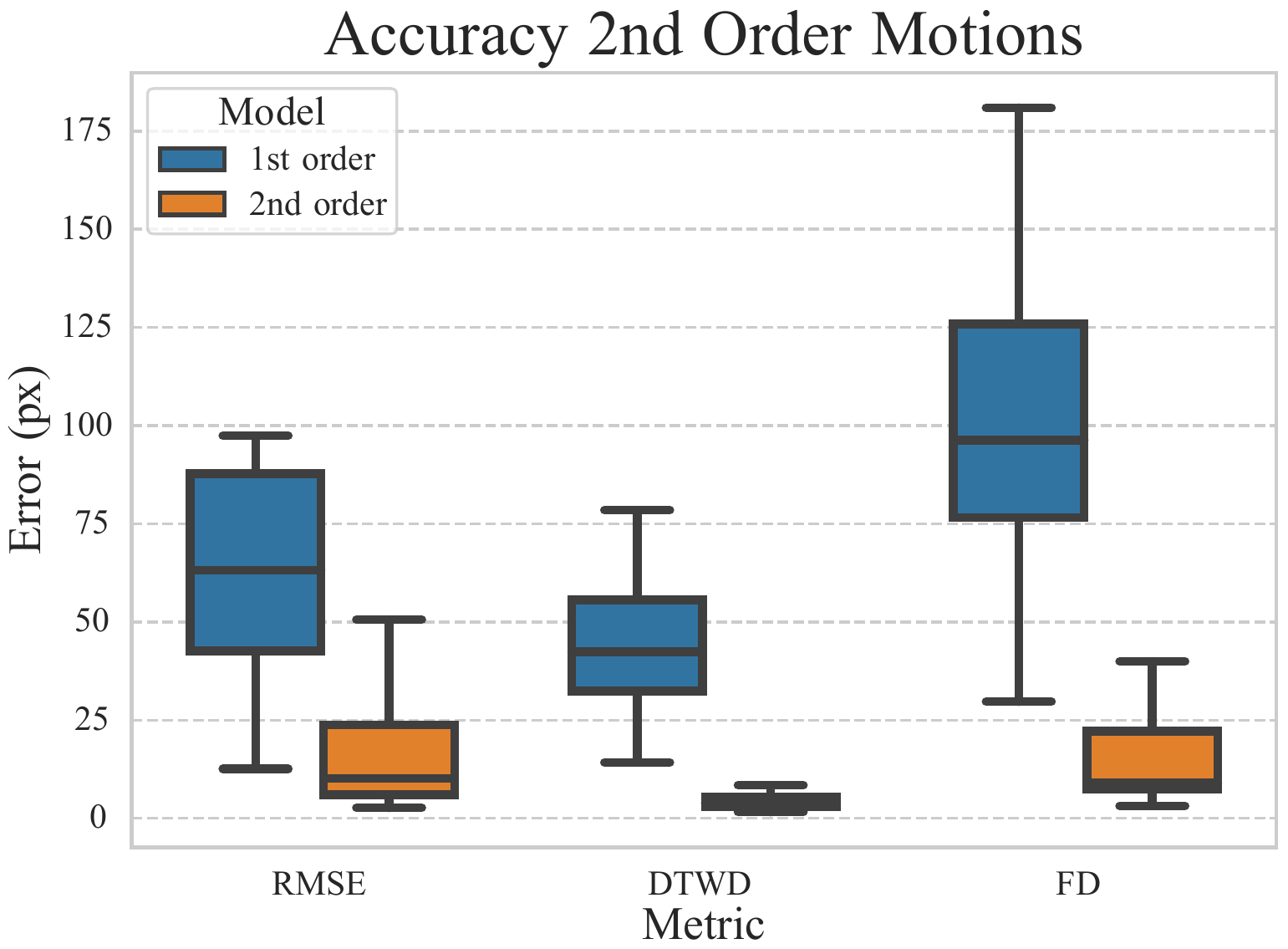}
    \caption{Accuracy comparison of CONDOR when modeling motions using first order and second order systems. Each box plot summarizes performance over the 10 motions of the LAIR dataset.}
    \label{fig:box_plot_second_order}
\end{figure}
\subsubsection{Stability}
\begin{table}[t]
\centering
\caption{Percentage of successful trajectories over the LAIR dataset ($L=2000$, $P=1225$, $\epsilon=10$px).}
\label{tab:stability_2}
\begin{tabular}{ll}
\hline
CONDOR (1st order) & CONDOR (2nd order) \\ \hline
9.3388\%     & \textbf{0.0000}\%                \\ \hline
\end{tabular}
\end{table}
Finally, we study the stability of the motions generated with CONDOR over the LAIR dataset when using first-order and second-order systems. Table \ref{tab:stability_2} shows that when using first-order systems, CONDOR struggles to generate stable motions with second-order demonstrations. For instance, when a demonstration has a \emph{loop}, the optimization of the DNN might not find a proper solution, since trajectories inside the loop do not have a way of reaching a region outside the loop without ignoring the demonstrations. In contrast, CONDOR with second-order systems is always able to learn stable motions.

\section{Real-World Experiments}\label{sec:real_world}

\begin{figure*}[t]
\centering
\subfloat[Hammer hanging.]{\includegraphics[width=0.32\linewidth,valign=t]{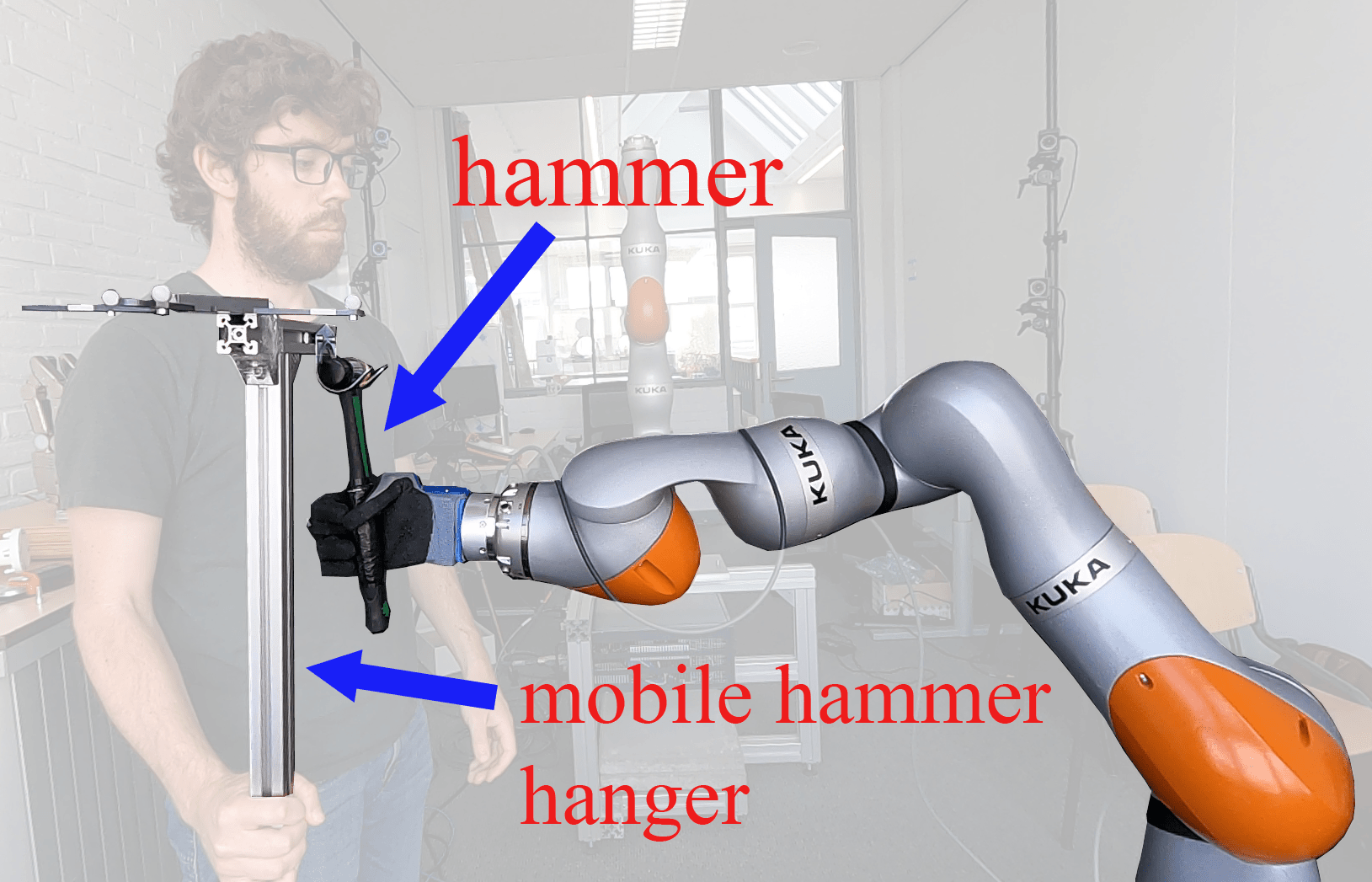}%
\vphantom{\includegraphics[width=0.32\linewidth,valign=t]{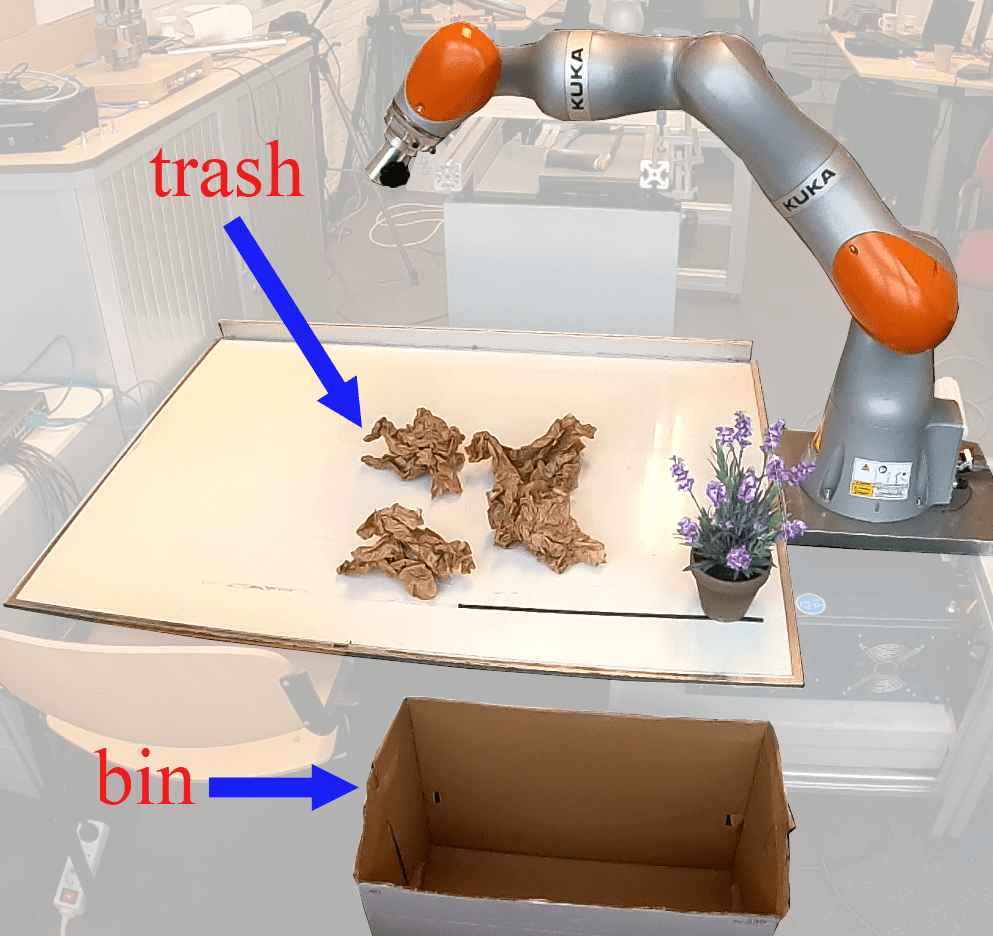}}%
\label{fig:hanging_setup}%
}
\quad
\subfloat[Writing number two.]{\includegraphics[width=0.32\linewidth,valign=t]{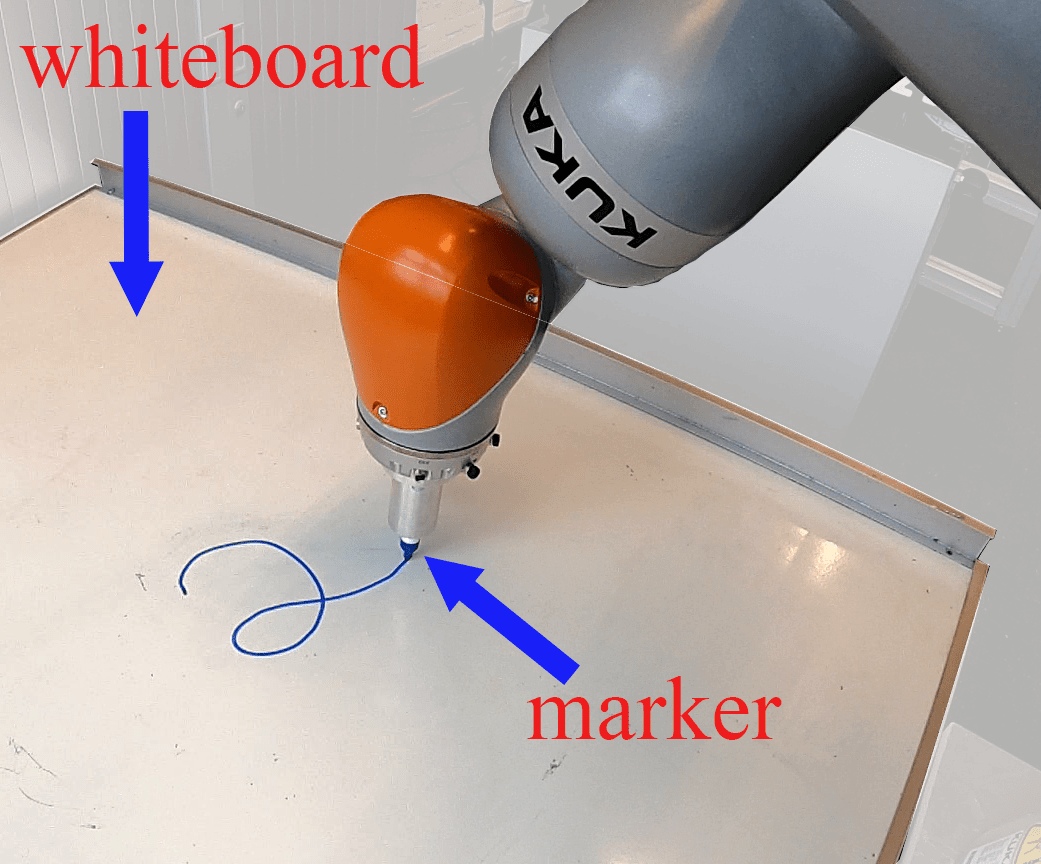}%
\vphantom{\includegraphics[width=0.32\linewidth,valign=t]{figures/cleaning_setup.png}}%
\label{fig:writing_setup}%
}
\quad
\subfloat[Table cleaning.]{\includegraphics[width=0.32\linewidth,valign=t]{figures/cleaning_setup.png}%
\label{fig:cleaning_setup}%
}
\caption{Setup of real-world experiments.}
\label{fig:realworld_setups}
\end{figure*}

\begin{table}[t]
\centering
\begin{threeparttable}
\caption{Characteristics of the real-world experiments.}
\label{tab:param_realworld}
\begin{tabular}{lcllll}
\hline
Task                                         & \multicolumn{1}{l}{\begin{tabular}[c]{@{}l@{}}Dim\\ state\end{tabular}} & \multicolumn{1}{l}{\begin{tabular}[c]{@{}l@{}}Order\\ motion\end{tabular}} & Control                                                            & Data collection & \begin{tabular}[c]{@{}l@{}} \#\\ demos\end{tabular}                                           \\ \hline
\begin{tabular}[c]{@{}l@{}}Hammer \\ hanging\end{tabular} & 3                                & \multicolumn{1}{c}{1}              & \begin{tabular}[c]{@{}l@{}}online,\\ end eff.\end{tabular}  & motion capture  &  10                                                    \\[0.4cm]
\multicolumn{1}{l}{\begin{tabular}[c]{@{}l@{}}Writing\\ two\end{tabular}}                                               & 4                                & \multicolumn{1}{c}{2}              & \begin{tabular}[c]{@{}l@{}}offline,\\ end eff.\end{tabular} & \begin{tabular}[c]{@{}l@{}}computer mouse\\ interface\end{tabular}  & 6 \\[0.4cm]
\begin{tabular}[c]{@{}l@{}}Table\\ cleaning\end{tabular}  & 6                                & \multicolumn{1}{c}{1}              & \begin{tabular}[c]{@{}l@{}}online,\\ joints\end{tabular}       & \begin{tabular}[c]{@{}l@{}}kinesthetic\\ teaching\end{tabular}    &  2$^{*}$                                              \\ \hline
\end{tabular}
\begin{tablenotes}
\footnotesize
\item $^{*}$ Two motion models were learned, and one demonstration was used per model.
\end{tablenotes}
\end{threeparttable}
\end{table}

To validate the proposed framework in more realistic scenarios, we design three real-world experiments using a 7-DoF KUKA iiwa manipulator: 1) hammer hanging, 2) writing the number two, and 3) cleaning a table (see Fig. \ref{fig:realworld_setups}). Throughout these experiments, four important characteristics of the learning problem are changed: 1) dimensionality of the motion, 2) order of the motion, 3) control strategy, and 4) data collection method. These characteristics define different Imitation Learning scenarios that can be found in real-world robotic problems. Hence, by testing CONDOR in these scenarios we aim to show the applicability, flexibility, and robustness of our method. Furthermore, if we compare these scenarios with the simulated ones studied in the previous sections, we can observe that our method is not restricted to 2-dimensional motions only and that it can also work in higher-dimensional problems. Table \ref{tab:param_realworld} shows a summary of the real-world experiments, which are explained in detail in the following subsections.

Similarly to the LAIR dataset, the demonstrations are not postprocessed in these experiments. Hence, in this section, the goal of the motions is also computed by taking the mean between the ending points of each demonstration.

\subsection{Hammer hanging: First-order 3D motions}
This experiment consists of learning to control the end-effector's position of a robot such that it hangs a hammer (see Fig. \ref{fig:hanging_setup}), allowing us to test the behavior of CONDOR for first-order 3-dimensional motions. This problem is interesting since it shows that implicit knowledge, that otherwise requires modeling, can be transferred to the robot via human demonstrations. In this case, this knowledge includes information about the geometry of the hammer and the hanger that is required to hang the hammer.

\begin{figure}[t]
    \centering
    \includegraphics[width=0.6\columnwidth]{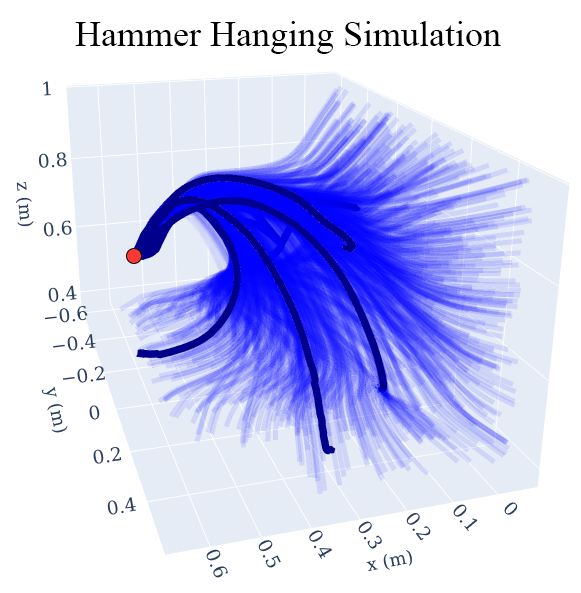}
    \caption{The blue trajectories represent the learned model's evolution of the robot end effector's position when starting from different initial conditions. The larger and darker trajectories correspond to demonstrations. Some demonstrations are occluded and others were removed for visualization purposes. The red point corresponds to the goal.}
    \label{fig:hammer_flow}
\end{figure}

\subsubsection{Control}
We employ the online control strategy depicted in Fig. \ref{fig:online_control}, which allows the robot to be reactive to perturbations and adjust its motion \emph{on the fly} if the environment changes. Hence, at every time step, the robot obtains its position with respect to the goal (i.e., the hanger) and sends an end effector's velocity request to a low-level controller. For details regarding the low-level controller, the reader is referred to Appendix \ref{sec:task_space_control_online}.

\subsubsection{Demonstrations}
We used a motion capture system to collect demonstrations. The demonstrator had to wear a glove whose position was tracked by the tracking system. We recorded 10 demonstrations and used them to train CONDOR. 

This approach has the advantage of it being comfortable for the human, since it does not require the human to adjust to any specific interface nor interact with the robot, which can require training. Nevertheless, since the robot embodiment is not being employed to collect the demonstrations, there is no guarantee that the collected motions will feasible for the robot to execute. Therefore, it is necessary to record motions that can be executed by the robot, which, depending on the problem, might require knowledge about the robotic platform.

\subsubsection{Moving goal}
To test the reactive capabilities of this approach, and the generalization properties of motions modeled as dynamical systems, we made this problem more challenging by making the hanger movable. To achieve this, we added tracking markers to the top of the hanger and fed the hanger's position to CONDOR in real time. Consequently, while the robot was executing the hanging motion, the hanger could be displaced and the robot had to react to these changes in the environment. 

Notably, no extra data is required to achieve this, since the motion of CONDOR is computed as a function of the relative position of the robot w.r.t. the goal. Hence, by displacing the goal, the position of the end effector with respect to the hanger changes, making CONDOR provide a velocity request according to this new position.

\subsubsection{Results}
Fig. \ref{fig:hammer_flow} shows a 3D plot with 1250 simulated trajectories generated with CONDOR when starting from different initial positions. We can observe that all of the trajectories reach the goal while following the shape in the demonstrations. The performance of this model on the real robot can be observed in the attached video.

\subsection{Writing: Second-order 2D motions}
We also tested CONDOR in a writing scenario (see Fig. \ref{fig:writing_setup}). The objective is to control the robot's end effector to write the number two on a whiteboard. To write the number two, it is necessary to use second-order motions, since this character has one intersection. Therefore, in this experiment, we aim to validate the ability of CONDOR for modeling second-order motions. Finally, note that for writing it is only necessary for the robot to move in a 2-dimensional plane; however, since the motion is of second order, the state space of the robot is 4-dimensional (the same as the motions in the LAIR dataset).

\subsubsection{Control}
We employ the offline control strategy as depicted in Fig. \ref{fig:offline_control}. This approach is suitable for writing since in this task it is important that the trajectory that the robot executes is consistent with the one that CONDOR predicts from the initial state. For instance, if the robot, while executing the motion is perturbed by its interaction with the whiteboard in some direction, it would transition to a state that is not consistent anymore with the character that has been written so far. In an offline control approach this is not very critical, because, given that the reference of the motion is pre-computed, it would make the robot move back to a state that is consistent with the motion that is being written. For details regarding the low-level controller, the reader is referred to Appendix \ref{sec:task_space_control_offline}.

\subsubsection{Demonstrations}
The same PC mouse interface developed to collect the LAIR dataset is employed here. 6 demonstrations were collected and used to train CONDOR. 

\subsubsection{Results}
The simulated results of this experiment follow the same behavior as the ones presented in Section \ref{sec:lair}. To observe its behavior on the real robot, the reader is referred to the attached video.

\subsection{Table cleaning: First-order 6D motions}
Finally, we test CONDOR in a cleaning task. The objective is to use the robot's arm to push garbage, which is on top of a table, to a trash bin. Differently from the other scenarios, in this case, the robot's joint space is directly controlled with CONDOR. Hence, we learn a 6-dimensional motion. Note that the robot has 7 degrees of freedom, but we keep the last joint fixed as it has no influence on the task. 

Since the motions learned by CONDOR can be used as primitives of a more complex motion, in this experiment we highlight this capability by learning two motions that are sequenced together to generate the complete cleaning behavior. Each motion is trained with only one demonstration. 

This scenario allows us to test two features of our method: 1) its behavior in a higher-dimensional space (6D), and 2) its capability to learn motions from only one demonstration.

\subsubsection{Control}
Similarly to the hanging hammer experiment, we use the online control strategy (Fig. \ref{fig:online_control}). Differently than before, in this case, the joint space of the robot is directly controlled with CONDOR, i.e., a reference velocity for the joints is provided to the low-level controller. Joint-space control is suitable for this task because the configuration of the robot is important for completing the task successfully since its body is used to push the trash. For details regarding the low-level controller, the reader is referred to Appendix \ref{sec:joint_space_control}.

\begin{figure}[t]
    \centering
    \includegraphics[width=1.0\columnwidth]{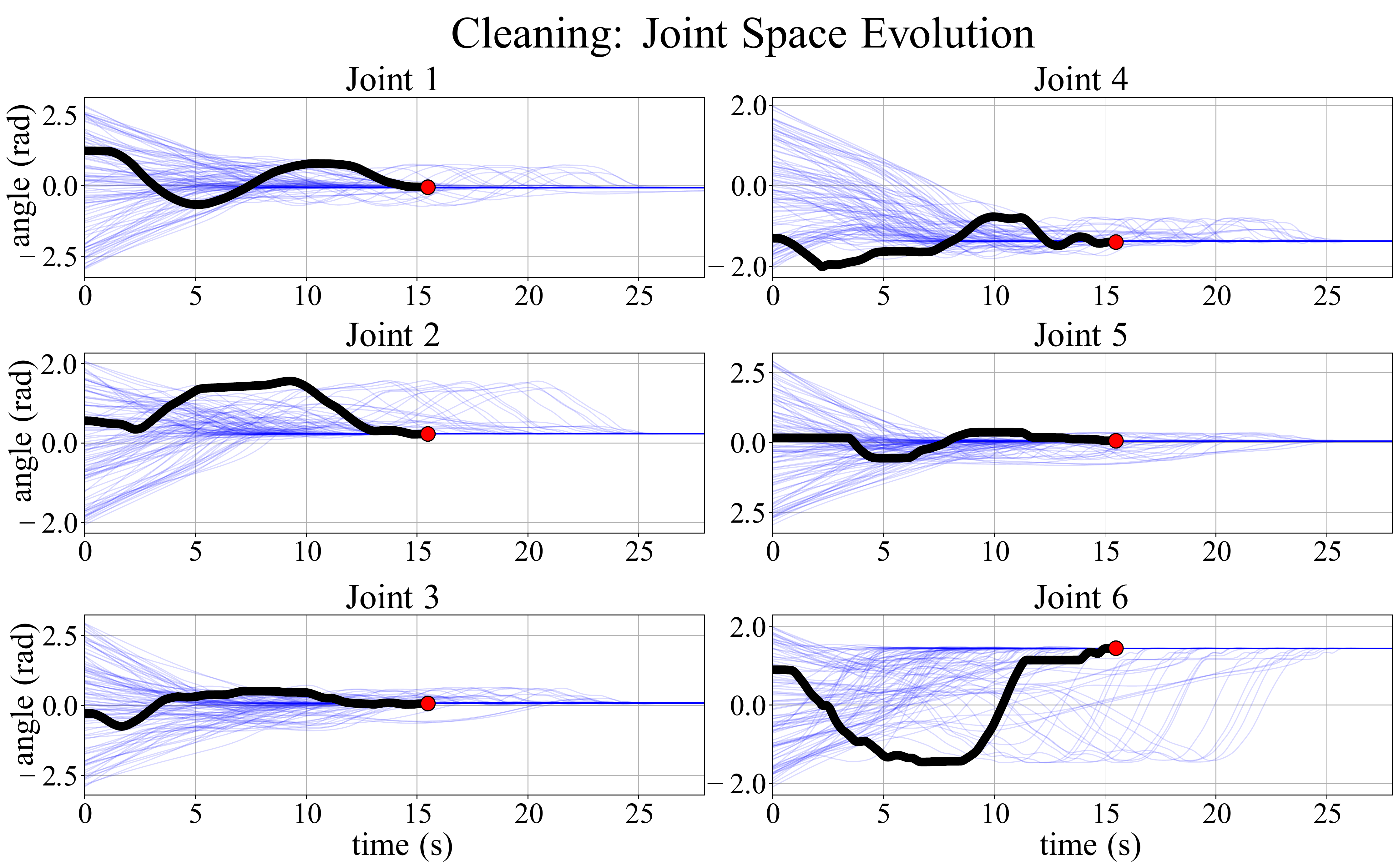}
    \caption{Simulated trajectories, as a function of time, of the second motion of the cleaning task. Blue trajectories correspond to evaluations of the model under different initial conditions and the black trajectory corresponds to the demonstration. The red point is the goal.}
    \label{fig:cleaning_flow}
\end{figure}

\subsubsection{Demonstrations}
For this experiment, kinesthetic teaching was used to collect demonstrations. This approach consists of collecting demonstrations by physically interacting with the robot and guiding it along the desired trajectory. To make this task easier, the gravitational forces of the robot were compensated such that it would not move unless the human interacted with it.

\subsubsection{Results}
Fig. \ref{fig:cleaning_flow} presents simulated results of the second cleaning primitive learned in this experiment. Since the motion is 6-dimensional and, hence, it is very challenging to visualize in one plot, we use six different plots to separately show the evolution of each state dimension as a function of time. 

In this case, we simulate 100 trajectories using CONDOR, since more make the plots difficult to analyze. From them, we observe that as time increases, every trajectory eventually reaches the goal. Note that, given that their initial states are random, they can start further away or closer to the goal than the demonstration; therefore, it might take them a longer/shorter time to reach the goal. Lastly, we can observe that the demonstrated trajectory and some simulations, either overlap or have the same shape with a phase shift, which showcases that the demonstrated behavior is captured by CONDOR.

The reader is referred to the attached video to observe the behavior of the cleaning primitives on the real robot.

\section{Extending CONDOR}\label{sec:extensions}
One of the advantages of our proposed framework is that we can extend it to address more complex problems. Therefore, there are interesting areas of research that can be studied with CONDOR. In this section, we aim to show the steps that we have taken in this direction, which we plan to study deeper in future work. More specifically, we tested two extensions:
\begin{enumerate}
    \item \textbf{Obstacle avoidance:} multiple obstacle avoidance methods have been proposed for motions modeled as dynamical systems, and it is an active field of research \cite{khansarizadeh2012avoidance,saveriano2014distance,huber2019avoidance,huber2022avoiding}. We test CONDOR with one of these extensions \cite{khansarizadeh2012avoidance} and observe that it works properly. However, apart from this validation, we do not provide a contribution to this problem. Hence, the reader is referred to Appendix \ref{sec:obstacle_avoidance} for more details regarding obstacle avoidance in this research. 
    \item \textbf{Multi-motion learning and interpolating:} another interesting field of research is learning multiple motions together in one Neural Network model. This allows interpolating between these motions, generating novel behaviors that are not present in the demonstrations. This can, for instance, reduce the number of human demonstrations required to learn and generalize a problem to a different situation.
\end{enumerate}

In the next subsection, we study the interpolation capabilities of motions learned with CONDOR.

\subsection{Multi-motion Learning and Interpolation}
We aim to provide preliminary results regarding the multi-motion learning capabilities of CONDOR, and its behavior in terms of interpolation and stability. To learn multiple motions in one Neural Network we extend its input with a one-hot code that indicates which motion is selected. For instance, if we learn three motions, we have three codes $[1, 0, 0]$, $[0, 1, 0]$, and $[0, 0, 1]$. Then, each code is used together with a different set of demonstrations to optimize $\ell_{\text{bc}}$. To interpolate between these motions, we select an input code of the DNN that has an intermediate value between the ones of the motions, e.g., $[0.5, 0.5, 0.0]$.

\begin{figure}[t]
    \centering
    \includegraphics[width=0.6\columnwidth]{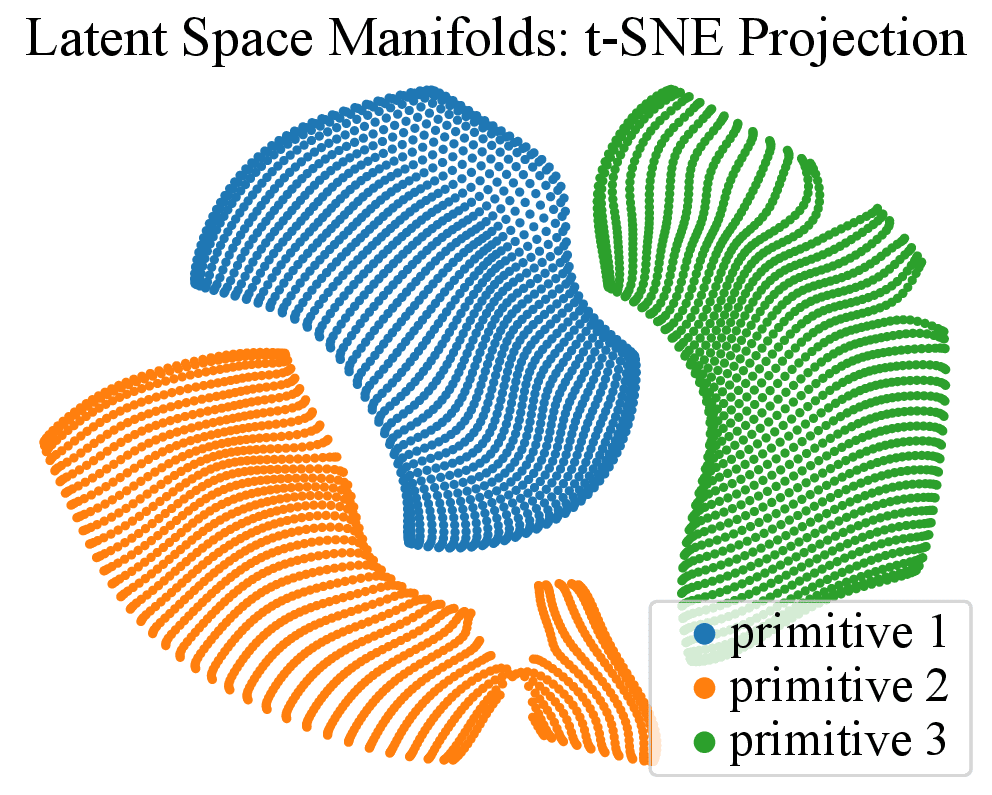}
    \caption{t-SNE projection of the motion manifolds present in the latent space of the DNN.}
    \label{fig:tsne}
\end{figure}
To ensure stability for all of the interpolated motions, we can minimize $\ell_{\text{stable}}$ for each motion and also for the ones in the interpolation space, which should create a bijective mapping between the complete input of the DNN (state and code) and its latent space. Part of this can be observed in Fig. \ref{fig:tsne}, which shows a t-SNE \cite{van2008visualizing} projection of three manifolds corresponding to the mapping of the state space of three motions to the DNN's latent space. Since each motion is mapped to a different region of the latent space, it is possible to move in between these regions to create interpolated motions.

Fig. \ref{fig:interpolation} shows these motions and some examples of their interpolation. We can note that the interpolation works properly, where features of different motions are combined to create novel behaviors. Furthermore, we observe that, as expected, the closer to a motion we interpolate, the more features of this motion the interpolated one showcase. Finally, regarding stability, every motion has zero unsuccessful trajectories with $L=2000$, $P=1225$, and $\epsilon=10$px.
\begin{figure}[t]
    \centering
    \includegraphics[width=\columnwidth]{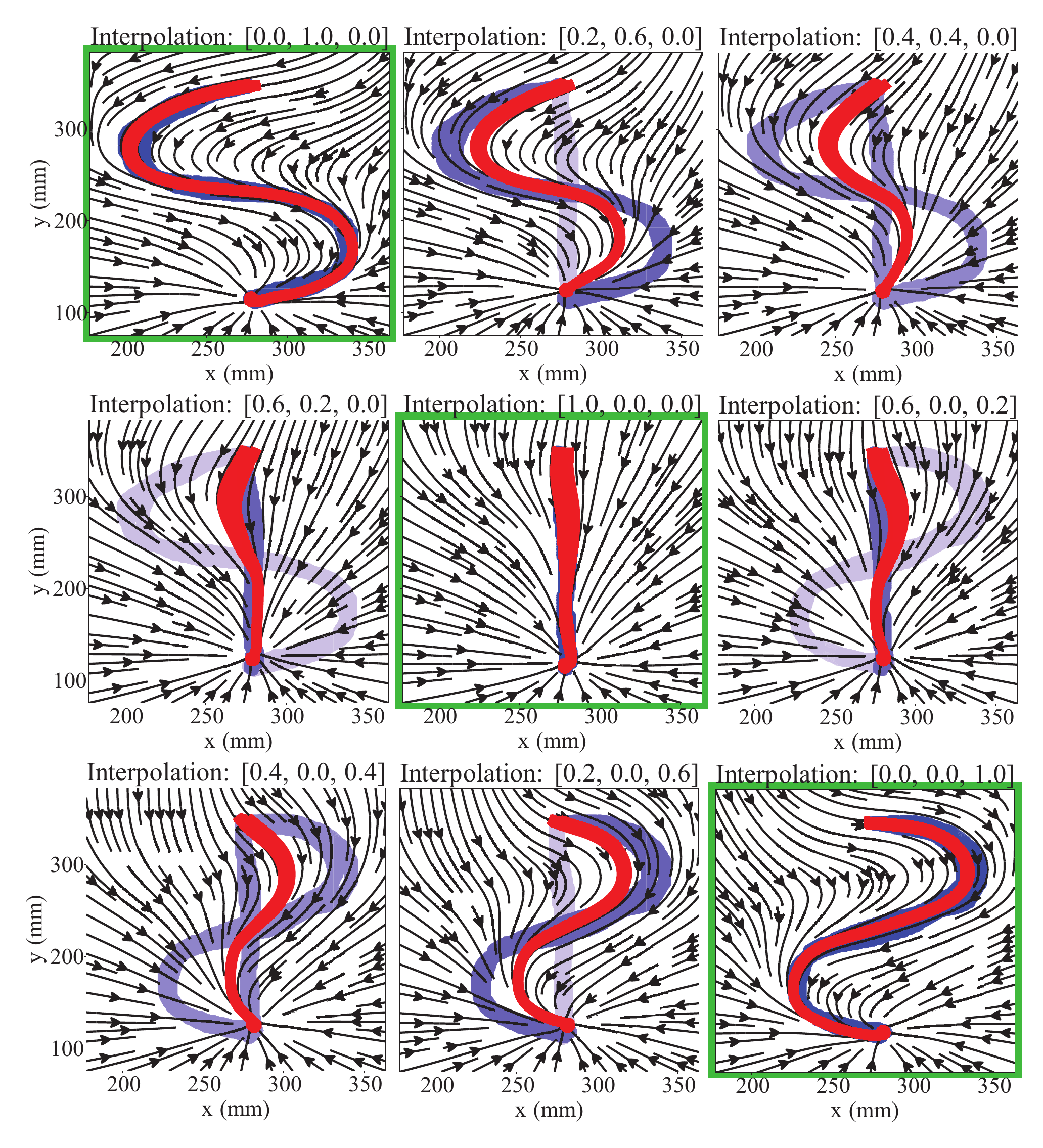}
    \caption{Different motions learned in one DNN. Blue curves correspond to demonstrations, the darker their color, the more they influence each motion. Arrows represent the vector field of the motion, and red curves correspond to simulations of trajectories executed by the model. The figures highlighted in green correspond to the cases where the BC loss is minimized. The remaining figures correspond to interpolated motions. In the title of each plot, we can observe the code provided to the network to generate each motion.}
    \label{fig:interpolation}
\end{figure}

\section{Conclusion}\label{sec:conclusions}
In the context of robotic reaching motions modeled as dynamical systems, we introduce a novel contrastive loss that extends current Imitation Learning frameworks to achieve globally asymptotically stable behaviors. We optimize this loss together with a Behavioral Cloning loss, which, despite its practical limitations due to the covariate shift problem, can achieve state-of-the-art results by minimizing the multi-step loss instead of the single-step loss, as observed in our experiments. Importantly, our stability loss can also be employed with other Imitation Learning approaches, though its effectiveness with other losses remains untested.

Further experiments demonstrate that our framework, CONDOR, can effectively learn stable and accurate motions across various scenarios. These experiments were conducted in both simulated settings and with a real robot. We observed that CONDOR learns successful behaviors in 1) 2-dimensional first-order motions (LASA dataset), 2) 3-dimensional first-order motions (hammer hanging), 3) 4-dimensional second-order motions (LAIR dataset and writing two), and 4) 6-dimensional motions (cleaning table). Lastly, we observe that CONDOR can be extended to learn multiple motions and interpolate between them, allowing it to generate more stable behaviors without requiring more demonstrations. This interesting area of research will be explored further in future work. 

While this paper's findings are promising, they also reveal limitations that inspire other future research directions. Firstly, our method has only been tested on relatively low-dimensional state spaces; its applicability in higher-dimensional spaces remains unexplored. Additionally, we assume that the employed state representations used are minimal, a condition not always met in robotics. For instance, orientation representations often employ non-Euclidean manifolds (e.g., unit quaternions or rotation matrices) that introduce constraints, making the state representations non-minimal \cite{ude2014orientation}. A further assumption is the existence of a low-level controller capable of generating the state transitions requested by CONDOR. While this is reasonable for manipulators, it can be a limiting factor in highly underactuated robots. Finally, CONDOR assumes that a proper state estimation (e.g., robot pose, goal location, or obstacles) is achieved by other modules.

\section*{Acknowledgment}
This research is funded by the Netherlands Organization for Scientific Research project Cognitive Robots for Flexible Agro-Food Technology, grant P17-01.

\appendices
\section{Approximating a Diffeomorphism}\label{appendix:diffeo_proof}
For achieving stable motions, CONDOR optimizes $\ell_{\text{stable}}$. This loss is designed to enforce the conditions of Theo. \ref{theo:stable} in the NN employed to represent the dynamical system $f_{\theta}^{\mathcal{T}}$. In this section, we show that as a consequence of this, $\psi_{\theta}$ approximates a diffeomorphism when $\mathcal{T}$ is a Euclidean space.

\begin{definition}[Diffeomorphism]
A mapping between two manifolds $\psi_{\theta}:\mathcal{T} \to \mathcal{L}$ is called a diffeomorphism if it is differentiable and bijective. 
\end{definition}

In general, we can consider Neural Networks to be differentiable, since most of the employed activation functions are differentiable. However, there are some exceptions to this, as it is for the case of the ReLU activation \cite{nair2010rectified}. In practice, these exceptions are non-differentiable at a small number of points and they have right and left derivatives, so they do not present many issues when computing their gradients. However, strictly speaking, for such cases, our method would make $\psi_{\theta}$ converge to a \emph{homeomorphism} instead. Differently to a diffeomorphism, a homeomorphism only requires the mapping to be continuous, but not differentiable. Nevertheless, for simplicity, we will assume that $\psi_{\theta}$ is differentiable.

Then, we need to study if $\psi_{\theta}$ converges to a bijective function to conclude that it approximates a diffeomorphism.
\begin{definition}[Bijective function]
A function is bijective if it is injective and surjective.
\end{definition}
\begin{definition}[Injective function]
 A function is injective if every distinct element of its domain maps to a distinct element, i.e.,  $\psi_{\theta}$ is injective if $\psi_{\theta}(x_{a})=\psi_{\theta}(x_{b})\Rightarrow x_{a}=x_{b}, \forall x \in \mathcal{T}$.
\end{definition}
\begin{definition}[Surjective function]\label{def:surjective}
 A function is surjective if every element of the function's codomain is the image of at least one element of its domain, i.e.,  $\forall y \in \mathcal{L}, \exists x \in \mathcal{T}\text{ such that } y=\psi_{\theta}(x)$.
\end{definition}

 From these definitions, it is clear that the surjectivity of $\psi_{\theta}$ is straightforward to show, since it depends on how its codomain is defined. In this work, we define the codomain of $\psi_{\theta}$ as $\mathcal{L}$, which is the manifold resulting from the image of $\psi_{\theta}$. In other words, the codomain $\mathcal{L}$ is a set that only contains the outputs of $\psi_{\theta}$ produced from $\mathcal{T}$. In such cases, a function is surjective, since its codomain and image are equal, which ensures that $\forall y \in \mathcal{L}, \exists x \in \mathcal{T}\text{ such that } y=\psi_{\theta}(x)$ (Definition \ref{def:surjective}).

 Consequently, it only remains to prove that if the conditions of Theo. \ref{theo:stable} are met, then $\psi_{\theta}$ is injective when $\mathcal{T}$ is a Euclidean space. 

\begin{proposition}\label{prop:injec}
If the conditions of Theo. \ref{theo:stable} are met and the domain of $\psi_{\theta}$ is $\mathcal{T}$; then, $\psi_{\theta}$ is injective.
\end{proposition}
\begin{proof}
By contradiction, let us assume that these conditions are met and $\psi_{\theta}$ is not injective. Then, let us take two elements of $\mathcal{T}$, $x_{a_{0}}$ and $x_{b_{0}}$, where $x_{a_{0}} \neq x_{b_{0}}$. If $\psi_{\theta}$ is not injective, there $\exists x_{a_{0}}$ and $\exists x_{b_{0}}$, such that $\psi_{\theta}(x_{a_{0}})=\psi_{\theta}(x_{b_{0}})$. In such cases, from Condition 1) of Theo. \ref{theo:stable}, we know that as $t\to \infty$, the mappings of these elements will generate the same trajectory in $\mathcal{L}$ following the dynamical system $f^{\mathcal{L}}$, which converges to $y_{g}$. 

Since the evolution of the variables $x_{a_{0}}$ and $x_{b_{0}}$ is completely defined by the evolution of their mappings in $\mathcal{L}$ (i.e., $\dot{x}=\phi(y)$), both variables will present the same time derivative in $\mathcal{T}$. Consequently, given that the trajectories obtained when starting from $x_{a_{0}}$ and $x_{b_{0}}$, at time $t$, are described by $x_{i}(t)=x_{i_{0}} + \int_{\tau=0}^{t}f(x(\tau))d\tau$\footnote{In discrete time, the integral transforms into a summation.}, where $i \in \{a, b\}$, their integral part will be the same $\forall t$. Then, $\forall t$ the distance between both trajectories is $d(x_{a_{0}}, x_{b_{0}})=||x_{a}(t) - x_{b}(t)||=||x_{a_{0}} - x_{b_{0}}||$, where $d(\cdot,\cdot)$ is a distance function. 

Thus, as $t\to \infty$, $x_{a}$ and $x_{b}$ will converge to two different points $x_{a_{g}}$ and $x_{b_{g}}$, respectively. However, we also stated that their respective mappings converge to $y_{g}$, i.e., $\psi_{\theta}(x_{a_{g}})=\psi_{\theta}(x_{b_{g}})=y_{g}$. In this case, Condition 2) of Theo. \ref{theo:stable} implies that $x_{a_{g}}=x_{b_{g}}=x_{g}$. This contradicts the fact that $x^{a}_{g}\neq x^{b}_{g}$. Consequently, $\psi_{\theta}$ is injective.
\end{proof}

Finally, from Prop. \ref{prop:injec} we can conclude that if the conditions of Theo. \ref{theo:stable} are enforced in $f^{\mathcal{T}}_{\theta}$; then, $\psi_{\theta}$ will approximate a diffeomorphism. 

\section{Stability of $f^{\mathcal{L}}$ with adaptive gains}\label{appendix:stability_adaptive}
In this section, we show that $y_{g}$ is globally asymptotically stable in the system introduced in \eqref{eq:DSy} when the adaptive gain $\alpha(y_{t})$ is greater than zero. Note that the derivation introduced here is analogous to the one of the discrete-time case when the system is simulated using the forward Euler integration method. However, in the latter, the condition for global asymptotic stability is $0 < \alpha(y_{t}) < 2 / \Delta t$. 

To show global asymptotic stability, we introduce the Lyapunov candidate $V(y_{t})=y_{t}^{\top}y_{t}$ and study if the condition $\dot{V}(y_{t}) < 0$ holds for all $y_{t} \in \mathbb{R}^{n}$. By introducing $A=\text{diag}(-\alpha(y_{t}))$ and, without loss of generality, setting $y_{g}=0$, we write \eqref{eq:DSy} as $\dot{y}_{t}=Ay_{t}$. Then,
\begin{equation}
\begin{split}
\dot{V}(y_{t}) &= (Ay_{t})^{\top}y_{t} + y_{t}^{\top}(Ay_{t}) \\
 &= y_{t}^{\top}(A + A^{\top})y_{t} \\
 &= 2 y_{t}^{\top}Ay_{t} \hspace{0.5cm}\text{(since $A$ diagonal)}.
\end{split}
\end{equation}
Therefore, it follows that this function is negative when the eigenvalues of $A$ are negative. Since $A$ is a diagonal matrix, the eigenvalues correspond to its diagonal $-\alpha(y_{t})$. Consequently, $y_{g}$ is globally asymptotically stable in the system \eqref{eq:DSy} when $\alpha(y_{t}) > 0$.

\section{Neural Network Architecture}\label{appendix:neural_network}
In this appendix, we provide details regarding the Neural Network's architecture. The criteria employed to design the architecture were to build a network: 1) large enough for it to be very flexible in terms of the motions that it can represent, and 2) with a reasonable size such that it can do inference in real time. Consequently, we observed that 3 feedforward fully connected layers, with 300 neurons each, for $\psi_{\theta}$ and $\phi_{\theta}$, i.e., 6 layers in total, were enough for obtaining accurate results and low inference times. The employed activation function was GELU \cite{hendrycks2016gaussian} for every layer except for the last layer of the network, which had a linear activation, and for the last layer of $\alpha$, which had a sigmoid function. In our case, the network inferred at $\pmb{677 \pm 57}$ \textbf{Hz} (confidence interval with one standard deviation) using a laptop PC with an Intel i7-8750H (12) @ 4.100GHz CPU and an NVIDIA GeForce RTX 2070 Mobile GPU. PyTorch had the GPU enabled at inference time. 

For the case of the adaptive gains $\alpha(y^{\mathcal{T}}_{t})$, two layers were employed instead. Note that these layers only affect the training time of the network, given that they are not required for inference. 

Finally, layer normalization \cite{ba2016layer} was added after each layer of the network except for the last layers of $\psi_{\theta}$, $\phi_{\theta}$ and $\alpha$. This type of normalization has shown to be beneficial for reducing training times and also for helping with vanishing gradients.

\begin{table*}[t]
\scriptsize
\caption{Hyperparameter optimization results of CONDOR.}
\label{tab:hyper_tuning_condor}
\begin{center}
\begin{tabular}{lclccccccccc}
\hline
\textbf{Hyperparameter}                                  & \textbf{Opt.?}                &  & \multicolumn{4}{c}{\textbf{\begin{tabular}[c]{@{}c@{}}Hand-tuned value / \\ initial opt. guess\end{tabular}}}                                                                                                                                                      & \textbf{}                     & \multicolumn{4}{c}{\textbf{Optimized value}}                                                                                                                                                                                                                       \\ \cline{4-7} \cline{9-12} 
\textbf{}                                                & \multicolumn{1}{l}{\textbf{}} &  & \multicolumn{1}{l}{\textbf{CONDOR}} & \textbf{\begin{tabular}[c]{@{}c@{}}CONDOR\\ (fixed gains)\end{tabular}} & \textbf{\begin{tabular}[c]{@{}c@{}}CONDOR\\ (triplet)\end{tabular}} & \textbf{\begin{tabular}[c]{@{}c@{}}CONDOR\\ (1st order\\ LAIR)\end{tabular}} & \multicolumn{1}{l}{\textbf{}} & \multicolumn{1}{l}{\textbf{CONDOR$^{***}$}} & \textbf{\begin{tabular}[c]{@{}c@{}}CONDOR\\ (fixed gains)\end{tabular}} & \textbf{\begin{tabular}[c]{@{}c@{}}CONDOR\\ (triplet)\end{tabular}} & \textbf{\begin{tabular}[c]{@{}c@{}}CONDOR\\ (1st order\\ LAIR)\end{tabular}} \\
\textbf{CONDOR}                                          & \multicolumn{1}{l}{}          &  & \multicolumn{1}{l}{}                & \multicolumn{1}{l}{}                                                    & \multicolumn{1}{l}{}                                                & \multicolumn{1}{l}{}                                                         & \multicolumn{1}{l}{}          & \multicolumn{1}{l}{}                & \multicolumn{1}{l}{}                                                    & \multicolumn{1}{l}{}                                                & \multicolumn{1}{l}{}                                                         \\ \hline
Max adap. latent gain ($\alpha_{\text{max}}$)            & \cmark                        &  & 1e-2                                & 8e-3$^{*}$                                                              & 1e-2                                                                & 9.997e-2                                                                     &                               & 9.997e-2                            & 2.470e-3$^{*}$                                                          & 3.970e-2                                                            & 0.174                                                                        \\
Stability loss weight ($\lambda$)                        & \cmark                        &  & 1                                   & 1                                                                       & 1                                                                   & 9.300e-2                                                                     &                               & 9.300e-2                            & 3.481                                                                   & 0.280                                                               & 2.633                                                                        \\
Window size imitation ($H^{i}$)                          & \cmark                        &  & 14                                  & 14                                                                      & 14                                                                  & 14                                                                           &                               & 14                                  & 14                                                                      & 14                                                                  & 1                                                                            \\
Window size stability ($H^{s}$)                          & \cmark                        &  & 4                                   & 4                                                                       & 2                                                                   & 4                                                                            &                               & 1                                   & 1                                                                       & 2                                                                   & 8                                                                            \\
Contrastive margin ($m$)                                 & \cmark                        &  & 1e-2                                & 1e-2                                                                    & 1e-4$^{**}$                                                         & 3.334e-2                                                                     &                               & 3.334e-2                            & 3.215e-3                                                                & 1.977e-4$^{**}$                                                     & 1.557e-2                                                                     \\
Batch size imitation ($B^{i}$)                           & \xmark                        &  & 250                                 & 250                                                                     & 250                                                                 & 250                                                                          &                               & -                                   & -                                                                       & -                                                                   & -                                                                            \\
Batch size stability ($B^{s}$)                           & \xmark                        &  & 250                                 & 250                                                                     & 250                                                                 & 250                                                                          &                               & -                                   & -                                                                       & -                                                                   & -                                                                            \\
                                                         & \multicolumn{1}{l}{}          &  & \multicolumn{1}{l}{}                & \multicolumn{1}{l}{}                                                    & \multicolumn{1}{l}{}                                                & \multicolumn{1}{l}{}                                                         & \multicolumn{1}{l}{}          & \multicolumn{1}{l}{}                & \multicolumn{1}{l}{}                                                    & \multicolumn{1}{l}{}                                                & \multicolumn{1}{l}{}                                                         \\
\textbf{Neural Network}                                  & \multicolumn{1}{l}{}          &  & \multicolumn{1}{l}{}                & \multicolumn{1}{l}{}                                                    & \multicolumn{1}{l}{}                                                & \multicolumn{1}{l}{}                                                         & \multicolumn{1}{l}{}          & \multicolumn{1}{l}{}                & \multicolumn{1}{l}{}                                                    & \multicolumn{1}{l}{}                                                & \multicolumn{1}{l}{}                                                         \\ \hline
Optimizer                                                & \xmark                        &  & AdamW                               & AdamW                                                                   & AdamW                                                               & AdamW                                                                        &                               & -                                   & -                                                                       & -                                                                   & -                                                                            \\
Number of iterations                                     & \xmark                        &  & 40000                               & 40000                                                                   & 40000                                                               & 40000                                                                        &                               & -                                   & -                                                                       & -                                                                   & -                                                                            \\
Learning rate                                            & \cmark                        &  & 1e-4                                & 1e-4                                                                    & 1e-4                                                                & 1e-4                                                                         &                               & 4.855e-4                              & 4.295e-4                                                                    & 8.057e-4                                                                & 5.553e-5                                                                       \\
Weight decay                                             & \xmark                        &  & 1e-4                                & 1e-4                                                                    & 1e-4                                                                & 1e-4                                                                            &                               & -                                   & -                                                                       & -                                                                   & -                                                                            \\
Activation function                                      & \xmark                        &  & GELU                                & GELU                                                                    & GELU                                                                & GELU                                                                         &                               & -                                   & -                                                                       & -                                                                   & -                                                                            \\
Num. layers ($\psi_{\theta}$, $\phi_{\theta}$, $\alpha$) & \xmark                        &  & (3, 3, 3)                           & (3, 3, -)                                                               & (3, 3, 3)                                                           & (3, 3, 3)                                                                    &                               & -                                   & -                                                                       & -                                                                   & -                                                                            \\
Neurons/hidden layer                                     & \xmark                        &  & 300                                 & 300                                                                     & 300                                                                 & 300                                                                          &                               & -                                   & -                                                                       & -                                                                   & -                                                                            \\
Layer normalization                                      & \xmark                        &  & yes                                 & yes                                                                     & yes                                                                 & yes                                                                          &                               & -                                   & -                                                                       & -                                                                   & -                                                                            \\ \hline
\end{tabular}
\end{center}
\begin{tablenotes}
\footnotesize
\item $^{*}$ Corresponds to the optimized fixed gain.
\item $^{**}$ Corresponds to the triplet loss margin.
\item $^{***}$ Also applies for CONDOR (2nd order).
\end{tablenotes}
\end{table*}
\section{Hyperparameter optimization} \label{appendix:hyerparameter_optimization}
We introduce a hyperparameter optimization strategy for CONDOR's different variations on the LASA and LAIR datasets. We define an accuracy metric, $\mathcal{L}_{\text{acc}}$, calculated using the distance metrics from Section \ref{sec:accuracy_description}. We also evaluate the stability of the system by minimizing the diffeomorphism mismatch, i.e., the RMSE between $y^{\mathcal{L}}_{1:\text{N}}$ and $y^{\mathcal{T}}_{1:\text{N}}$, defining the stability term, $\mathcal{L}_{\text{stable}}$. Lastly, we account for the precision of the learned system's goal versus the real goal by measuring the average distance of all final trajectory points to the goal, creating the term $\mathcal{L}_{\text{goal}}$. Then, we define the following objective:
\begin{equation}
    \mathcal{L}_{\text{hyper}} = \mathcal{L}_{\text{acc}} + \gamma_{\text{stable}} \mathcal{L}_{\text{stable}} + \gamma_{\text{goal}} \mathcal{L}_{\text{goal}},
\end{equation}
where $\gamma_{\text{stable}}$ and $\gamma_{\text{goal}}$ are weighting factors. After initial tests, we settled on $\gamma_{\text{stable}}=0.48$ and $\gamma_{\text{goal}}=3.5$.

In practical applications, hyperparameter tuning has limitations like time consumption and susceptibility to the curse of dimensionality. To mitigate this, we focused on five strategies: reducing the objective function's overhead, limiting the evaluation set, employing Bayesian optimization, pruning, and selecting a subset of hyperparameters.

\subsubsection{Reduced objective function's overhead}
The objective function minimized in the hyperparameter optimization process is periodically computed throughout each learning process. Consequently, if this function is expensive to compute, it will make the optimization process slower. More specifically, we observe that the computation of the accuracy using the DTWD and FD metrics adds considerable overhead to the computation time of the objective function. Furthermore, we also observe that the values of the RMSE, DTWD, and FD are highly correlated. Therefore, since computing the RMSE is much faster than computing the other metrics, the hyperparameter optimization loss that accounts for accuracy only consists of the RMSE, i.e., $\ell_{\text{IL}}$ with $H^{i}=n$ and $t'=0$.

\subsubsection{Reduced evaluation set}
Optimizing hyperparameters for the LASA/LAIR dataset using different motions simultaneously is challenging since the objective computed from different motions is not comparable. Instead, we focused on optimizing using a single, \emph{difficult} motion, assuming robust hyperparameters for it would perform well overall. We selected the \emph{heee} motion from the LASA dataset for first-order motions and the \emph{capricorn} motion from the LAIR dataset for second-order motions. These motions, with complex features like large curvatures or sharp edges, represented challenging test cases.

\subsubsection{Bayesian optimization}
During optimization, every evaluation of a different set of hyperparameters is costly. Therefore, instead of randomly selecting the hyperparameters to evaluate at each run or following a grid search approach, we select the most promising set given the ones evaluated so far. To achieve this, we employ the Tree Parzen Estimator (TPE) \cite{bergstra2011algorithms}, which builds a probability model of the objective function and uses it to select the next set of hyperparameters based on how promising they are. We use the implementation available in the Optuna API \cite{akiba2019optuna}.

\subsubsection{Pruning}
Throughout the optimization process, it is possible to detect inauspicious runs after a few evaluations. Hence, these trials can be \emph{pruned} before the training process ends, freeing the computational resources for a new run to be executed. We also incorporate this feature in the optimization process using the pruning method available in the Optuna API.

\subsubsection{Select a subset of hyperparameters}
Finally, before starting the optimization process, by interacting with CONDOR, we identified a subset of the hyperparameters that showed to have the largest influence over its results. Therefore, to reduce the dimensionality of the search problem, only this subset of hyperparameters is optimized. The rest are manually tuned based on our interactions with the framework. 

\subsection{Results}
Table \ref{tab:hyper_tuning_condor} details the results of the hyperparameters optimization process, including the optimized parameters, and their pre- and post-optimization values. It is divided into two sections: hyperparameters specific to CONDOR, and those general to DNNs. Note that most optimized hyperparameters pertain to the CONDOR method. For the LAIR dataset, we used LASA's optimized hyperparameters as a starting point, resulting in no improved set found for the second-order CONDOR method. Hyperparameters used in BC are excluded as those applicable were identical to those of CONDOR.

Note that the hyperparameter $\alpha_{\text{max}} \in (0,1]$ has not been introduced yet. This hyperparameter limits the maximum value of the adaptive gain $\alpha$ in $f^{\mathcal{L}}$ (see Section \ref{sec:adaptive_gains}). Hence, if in this work we define $\Delta t=1$ for $\alpha$ in $f^{\mathcal{L}}$; then, even though its maximum allowed value is 1, we limit it even further using $\alpha_{\text{max}}$. This process improves CONDOR's performance, as observed in preliminary experiments. Hence, we define $\alpha=\alpha_{\max}\bar{\alpha}(y_{t}^{\mathcal{T}})$, with $\bar{\alpha}$ as the DNN output using the sigmoid activation function, ensuring $\alpha \in (0, \alpha_{\max})$.

\section{Real-World Experiments: low-level control}\label{sec:low_level_control}
Regarding the low-level control strategy employed in this work, we focus on fully actuated rigid body dynamics systems, which evolve according to the following equation of motion:
\begin{equation}
    M(q)\ddot{q} + C(q,\dot{q})\dot{q} + G(q) + D(q)\dot{q} = u,
\end{equation}
where $q$ is the joint angle vector, $M(q)$ is the inertia matrix, $C(q, \dot{q})$ is the Coriolis/centripetal vector, $G(q)$ is the gravity vector, $D(q)$ is the viscous friction matrix and $u$ is the actuation torque vector \cite{Siciliano2009}. 

The objective of the low-level controller is to, at every time step, map the desired state $x^{d}$ and state derivative $\dot{x}^{d}$ provided by CONDOR to $u$, such that the system is driven towards the desired state. In our experiments, we learn motions in task space and in joint space. Hence, we employ slightly different strategies for each case.

The control frequency of the low-level controller is $\pmb{500}$ \textbf{Hz} in every experiment.

\subsection{Joint space control}\label{sec:joint_space_control}
In this work, independently of the task that CONDOR controls, every motion reference is eventually mapped to joint space. Hence, this subsection explains our approach to track this reference in joint space $(q^{d}, \dot{q}^{d})$.  We achieve this by means of a proportional-derivative (PD) controller with gravity compensation, i.e.,
\begin{equation}
    u = \alpha(q^{d} - q) + \beta(\dot{q}^{d}-\dot{q}) + G(q),
\label{eq:PD_control}
\end{equation}
where $\alpha$ and $\beta$ are gain matrices. The higher the gains of this controller, the smaller the tracking error \cite{kawamura1988local,della2021model}. Moreover, an interesting property of this approach is that as $q^{d}$ approaches $q_{g}$, where $q_{g}$ corresponds to the mapping from $x_{g}$ to the configuration space of the robot, CONDOR makes $\dot{x}^{d}$, and in consequence $\dot{q}^{d}$, tend to $0$. Then, \eqref{eq:PD_control} behaves similarly to
\begin{equation}
    u = \alpha(q^{d} - q) - \beta\dot{q} + G(q).
\label{eq:PD_control_reduced}
\end{equation}
This control law ensures global asymptotic stability at the equilibrium $q^{d}$ for any choice of $\alpha$ and $\beta$ as long as these are positive definite matrices \cite{Siciliano2009}.

\subsection{Task space control: Online}\label{sec:task_space_control_online}
To control the robot when references $\dot{x}^{d}$ are given online (see Fig. \ref{fig:online_control}) in task space, we use the real-time Inverse Kinematics (IK) library TRACK-IK \cite{beeson2015trac}. To do so, we integrate the velocity reference using the forward Euler integrator to obtain $x^{d}$, and we map this position to joint space using this library to obtain $q^{d}$. We apply exponential smoothing to these results to alleviate vibrations and stuttering issues. Finally, we compute the desired velocity $\dot{q}^{d}$ using the forward difference of $q$, i.e., $\dot{q}^d=(q^{d}-q)/\Delta t$, where $\Delta t$ is the time step length of CONDOR. Then, at every time step, the values $q^{d}, \dot{q}^{d}$ are provided to the controller described in Appendix \ref{sec:joint_space_control}.

\subsection{Task space control: Offline}\label{sec:task_space_control_offline}
In the offline case (see Fig. \ref{fig:offline_control}), a trajectory in task space $(x^{d}_{0}, x^{d}_{1},...,x^{d}_{N})$ is first computed with CONDOR. This trajectory is fed to the low-level controller to execute it offline. To achieve this, firstly, we map the trajectory to joint space using the Levenberg-Marquadt IK solver of the Robotics Toolbox \cite{corke2021rtb}. In this case, we employ this solver instead of TRACK-IK, because it is robust around singularities and can avoid problems like stuttering \cite{love2020levenberg}. This is important for obtaining very smooth solutions in scenarios where this is critical, such as in writing tasks. Note that the solver does not run in real time; however, this is not problematic, since the IK solutions are computed offline.

Afterward, the resulting joint space reference trajectory is approximated with a spline \cite{dierckx1995curve} that evolves as a function of time. Finally, when the motion starts, the time is incremented by $\Delta t$ at each time step and used to query the reference value that is sent to the controller described in Appendix \ref{sec:joint_space_control}.

\begin{figure}[t]
    \centering
    \includegraphics[width=\linewidth]{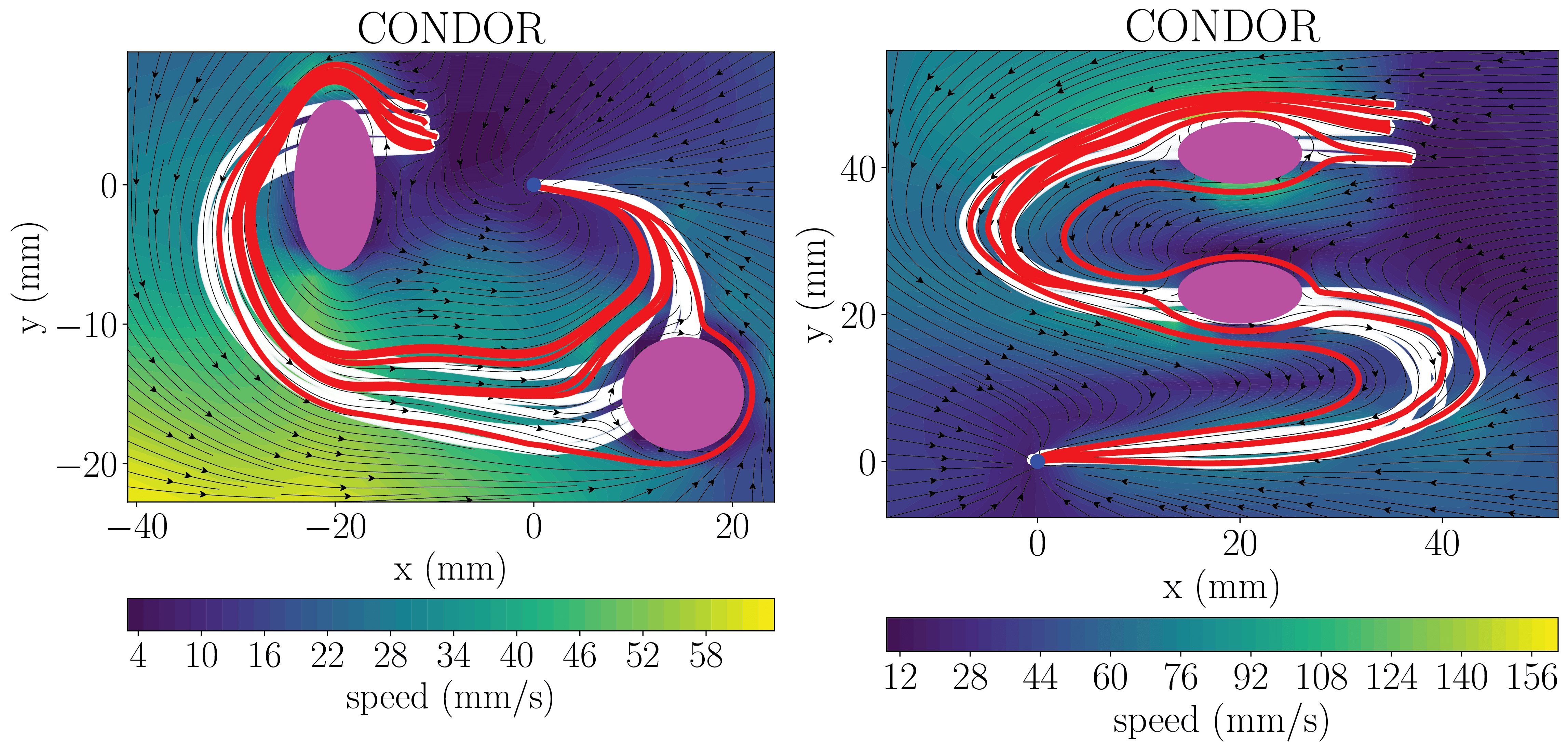}
    \caption{Obstacle avoidance in the LASA dataset.}
    \label{fig:LASAobstacle}
\end{figure}

\section{Obstacle avoidance}\label{sec:obstacle_avoidance}
Obstacle avoidance for motions modeled as dynamical systems is a problem that has been addressed in the literature \cite{khansarizadeh2012avoidance,severiano2013avoidance,huber2019avoidance}. Any of these methods can be combined with our proposed framework. In this work, we implemented the method presented in \cite{khansarizadeh2012avoidance} in PyTorch, and combined it with CONDOR. We compute a modulation matrix $M(x)$ that, when multiplied with the learned dynamical system $f(x)$, modifies the motion such that a new dynamical system $\bar{f}(x)=M(x)f(x)$ is obtained. $\bar{f}(x)$ avoids obstacles while maintaining the stability properties of $f(x)$. For more details please refer to \cite{khansarizadeh2012avoidance} (obstacle avoidance of multiple convex obstacles).

Fig. \ref{fig:LASAobstacle} shows motions from the LASA dataset where we test this approach. We observe that the motions generated with CONDOR remain stable after applying the modulation matrix. Furthermore, the obstacles are successfully avoided, showing that, as expected, the dynamical system motion formulation of CONDOR can be effectively combined with methods designed to work with dynamical systems. 

Finally, this method was tested with 3D obstacles in a real 7DoF robot manipulator when controlling its end effector position. These results are provided in the attached video.

\bibliographystyle{IEEEtran}
\bibliography{IEEE_bibliography}



\begin{IEEEbiography}[{\includegraphics[width=1in,height=1.25in,clip,keepaspectratio]{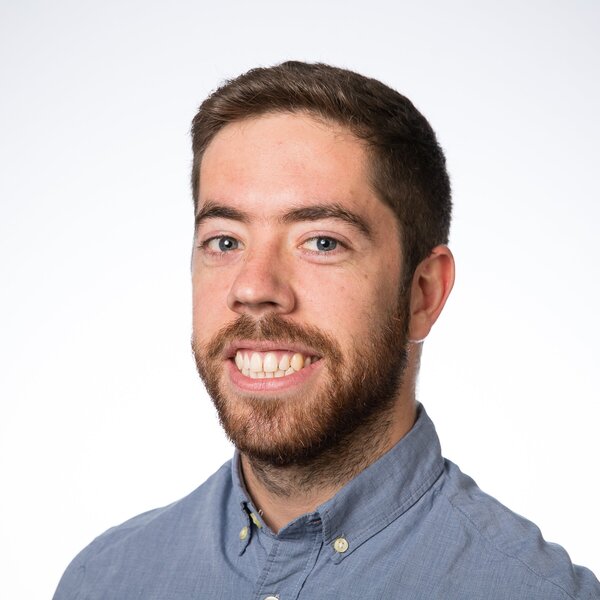}}]{Rodrigo Pérez-Dattari}
received the B.Sc. and M.Sc. degrees in Electrical Engineering at the University of Chile, Santiago, Chile, in 2017 and 2019, respectively. During this time, he was a member of the UChile Robotics Team, where he participated in the RoboCup competition in the years 2016-2018. He is currently pursuing a Ph.D. degree in robotics at the Delft University of Technology. He is part of the FlexCRAFT project, a Dutch research program to advance robotics for flexible agro-food technology. His research interests include imitation learning, interactive learning, and machine learning for control.
\end{IEEEbiography}

\begin{IEEEbiography}[{\includegraphics[width=1in,height=1.25in,clip,keepaspectratio]{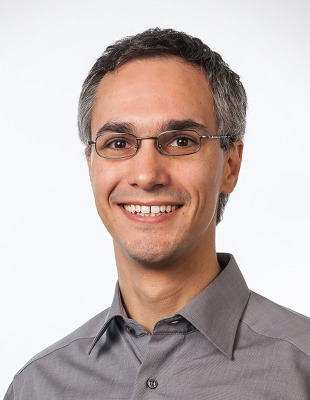}}]{Jens Kober}
is an associate professor at the TU Delft, Netherlands. He worked as a postdoctoral scholar jointly at the CoR-Lab, Bielefeld University, Germany and at the Honda Research Institute Europe, Germany. He graduated in 2012 with a PhD Degree in Engineering from TU Darmstadt and the MPI for Intelligent Systems. For his research he received the annually awarded Georges Giralt PhD Award for the best PhD thesis in robotics in Europe, the 2018 IEEE RAS Early Academic Career Award, the 2022 RSS Early Career Award, and has received an ERC Starting grant. His research interests include motor skill learning, (deep) reinforcement learning, imitation learning, interactive learning, and machine learning for control.
\end{IEEEbiography}

\end{document}